\documentclass{article}

\usepackage{microtype}
\usepackage{graphicx,flushend}
\usepackage{booktabs} 

\usepackage{hyperref}

 \usepackage[accepted]{icml2024}

\usepackage{amsmath}
\usepackage{amssymb}
\usepackage{mathtools}
\usepackage{amsthm}

\usepackage[capitalize,noabbrev]{cleveref}

\theoremstyle{plain}
\newtheorem{theorem}{Theorem}[section]
\newtheorem{proposition}[theorem]{Proposition}
\newtheorem{lemma}[theorem]{Lemma}
\newtheorem{corollary}[theorem]{Corollary}
\theoremstyle{definition}
\newtheorem{definition}[theorem]{Definition}
\newtheorem{assumption}[theorem]{Assumption}
\theoremstyle{remark}
\newtheorem{remark}[theorem]{Remark}

\usepackage[textsize=tiny]{todonotes}

\usepackage{tocloft}
\usepackage{titletoc}
\usepackage{appendix}

\usepackage{breakcites}
\usepackage{authblk}
\usepackage{enumitem}

\usepackage{hyperref}

\hypersetup{
     colorlinks = true,
     linkcolor = brown,
     anchorcolor = blue,
     citecolor = teal,
     filecolor = blue,
     urlcolor = black
     }

\usepackage{amsmath,amsfonts,bm}

\def\1{\bm{1}}

\DeclareMathAlphabet{\mathsfit}{\encodingdefault}{\sfdefault}{m}{sl}
\SetMathAlphabet{\mathsfit}{bold}{\encodingdefault}{\sfdefault}{bx}{n}

\usepackage{caption}
\usepackage{booktabs}
\usepackage{multirow}
\usepackage{enumitem}

\usepackage{subcaption}
\usepackage{hyperref}
\usepackage{url}
\usepackage{algorithm,algorithmic}
\usepackage{amsmath,amsthm}
\usepackage{amssymb,bm}
\usepackage{bbm}
\usepackage{makecell}

\newtheorem{assum}{Assumption}
\usepackage{pgfplots}
\pgfplotsset{compat=1.10}
\usetikzlibrary{
        patterns,
        pgfplots.fillbetween,
    }

       \tikzset{
        hatch distance/.store in=\hatchdistance,
        hatch distance=10pt,
        hatch thickness/.store in=\hatchthickness,
        hatch thickness=2pt,
        hatch color/.store in=\hatchcolor,      
        hatch color=black,                      
    }
    \pgfdeclarepatternformonly[%
        \hatchdistance,%
        \hatchthickness,%
        \hatchcolor
            ]{flexible hatch}
    {\pgfqpoint{0pt}{0pt}}
    {\pgfqpoint{\hatchdistance}{\hatchdistance}}
    {\pgfpoint{\hatchdistance-1pt}{\hatchdistance-1pt}}%
    {
        \pgfsetlinewidth{\hatchthickness}
        \pgfpathmoveto{\pgfqpoint{0pt}{0pt}}
        \pgfpathlineto{\pgfqpoint{\hatchdistance}{\hatchdistance}}
        \pgfsetstrokecolor{\hatchcolor}
        \pgfusepath{stroke}
    }

\newcommand{\youssef}[1]{\textbf{\textcolor{blue}{Youssef: #1}}}

\usepackage[makeroom]{cancel}

\hypersetup{
     colorlinks = true,
     linkcolor = red,
     anchorcolor = blue,
     citecolor = teal,
     filecolor = blue,
     urlcolor = black
     }

\newcommand\rankone[1]{\colorbox{lightgray}{\textcolor{black}{\textbf{#1}}}}
\newcommand\ranktwo[1]{\colorbox{lightgray}{\textcolor{black}{\textbf{#1}}}}
\newcommand\rankthree[1]{\colorbox{lightgray}{\textcolor{black}{\textbf{#1}}}}

\author[a]{Apoorva Nitsure}
\author[a]{Youssef Mroueh}
\author[a]{Mattia Rigotti}
\author[a,b]{Kristjan Greenewald}
\author[a]{Brian Belgodere}
\author[b]{Mikhail Yurochkin}
\author[a]{Jiri Navratil}
\author[a]{Igor Melnyk}
\author[a]{Jerret Ross}
\affil[a]{IBM Research}
\affil[b]{MIT-IBM Watson AI Lab}

 \usepackage{tcolorbox}
 \newtcolorbox{assbox}{colback=black!5!white,colframe=black!75!black}
  \newtcolorbox{thmbox}{colback=blue!5!white,colframe=black!75!black}

\icmltitlerunning{Risk Aware Benchmarking of Large Language Models}

\begin{document}

\twocolumn[
\icmltitle{Risk Aware Benchmarking of Large Language Models}



\icmlsetsymbol{equal}{*}

\begin{icmlauthorlist}
\icmlauthor{Apoorva Nitsure}{yyy}
\icmlauthor{Youssef Mroueh}{yyy}
\icmlauthor{Mattia Rigotti}{yyy}
\icmlauthor{Kristjan Greenewald}{yyy,comp}
\icmlauthor{Brian Belgodere}{yyy}
\icmlauthor{Mikhail Yurochkin}{yyy,comp}
\icmlauthor{Jiri Navratil}{yyy}
\icmlauthor{Igor Melnyk}{yyy}
\icmlauthor{Jarret Ross}{yyy}
\end{icmlauthorlist}

\icmlaffiliation{yyy}{IBM Research}
\icmlaffiliation{comp}{MIT-IBM Watson AI Lab}

\icmlcorrespondingauthor{Apoorva Nitsure}{Apoorva.Nitsure@ibm.com}
\icmlcorrespondingauthor{Youssef Mroueh}{mroueh@us.ibm.com}

\icmlkeywords{Machine Learning, ICML}

\vskip 0.3in
]



\printAffiliationsAndNotice{} 

\begin{abstract}
We propose a distributional framework for benchmarking socio-technical risks of foundation models with quantified statistical significance. Our approach hinges on a new statistical relative testing  based on first and second order stochastic dominance of real random variables. We show that the second order statistics in this test are linked to mean-risk models commonly used in econometrics and mathematical finance to balance risk and utility when choosing between alternatives. Using this framework, we formally develop a risk-aware approach for foundation model selection given guardrails quantified by specified metrics. Inspired by portfolio optimization and selection theory in mathematical finance, we define a \emph{metrics portfolio} for each model as a means to aggregate a collection of metrics, and perform model selection based on the stochastic dominance of these portfolios. The statistical significance of our tests is backed theoretically by an  asymptotic analysis via central limit theorems instantiated in practice via  a bootstrap variance estimate. 
We use our framework to compare various large language models regarding risks related to drifting from instructions and outputting toxic content.    
\end{abstract}


\section{Introduction }

Foundation models such as large language models (LLMs) have shown remarkable capabilities redefining the field of artificial intelligence. At the same time,  they present  pressing and challenging socio-technical risks regarding the trustworthiness of their outputs and their alignment with human values and ethics \citep{bommasani2021opportunities}. Evaluating LLMs is therefore a multi-dimensional problem, where those risks are benchmarked across diverse tasks and domains \citep{chang2023survey}.

In order to quantify these risks, \cite{liang2022holistic,wang2023decodingtrust,huang2023trustgpt,sun2024trustllm} proposed  benchmarks  of automatic metrics  for probing the trustworthiness of  LLMs. These metrics include accuracy, robustness, fairness, toxicity of the outputs, etc.  Human evaluation benchmarks can be even more nuanced, and are often employed when tasks surpass the scope of standard metrics. Notable benchmarks based on human and automatic evaluations include, among others, Chatbot Arena \citep{zheng2023judging}, HELM \citep{bommasani2023holistic}, MosaicML’s Eval, Open LLM Leaderboard \citep{open-llm-leaderboard}, and BIG-bench \citep{srivastava2022beyond}, each catering to specific evaluation areas such as chatbot performance, knowledge assessment, and domain-specific challenges. Traditional metrics, however, sometimes do not correlate well with human judgments. Aiming for a better alignment with human judgments, some approaches utilize ChatGPT/GPT-4 for natural language generation evaluations \citep{liu2023gpteval, zhang2023summit, hada2023large}.



\begin{figure*}[t!]
\centering
           \includegraphics[scale=0.5]{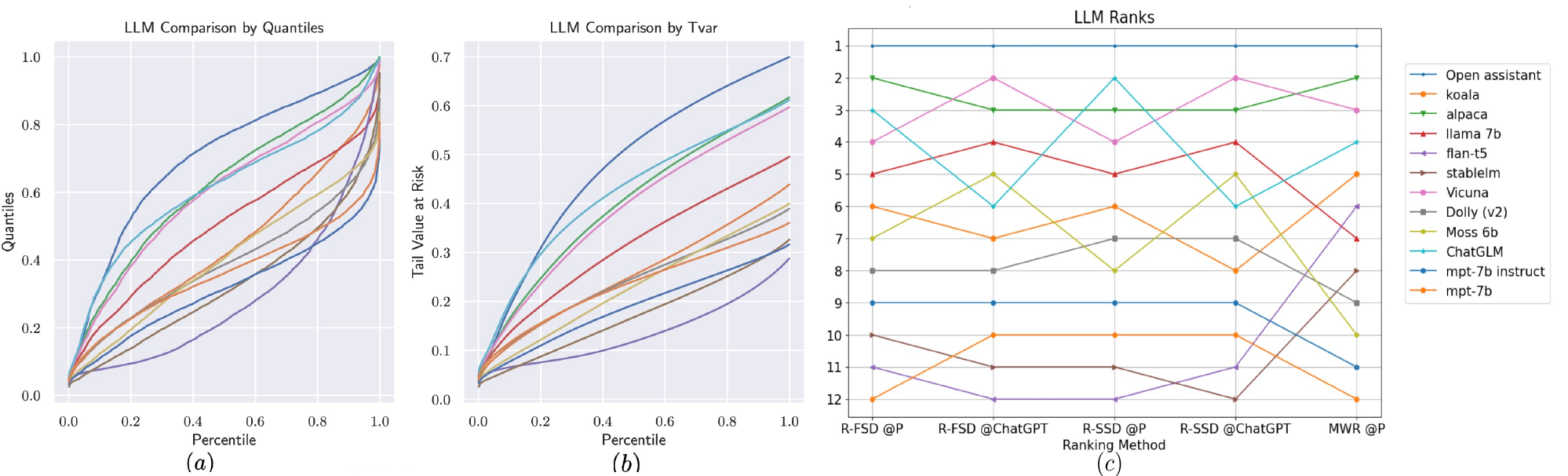} 
           \caption{(a) Quantiles, (b) Tail Value at Risk (TVAR), of Metrics portfolio of an LLM, showing that TVAR (second-order stochastic dominance) more clearly ranks the models than the quantiles alone (first-order stochastic dominance). (c) Ranking of models using Relative First and Second Stochastic Dominance of Portfolios (R-FSD, R-SSD @P) versus ranking of models using Relative First and Second Stochastic Dominance of chatGPT evaluation scores and ranking by Mean Win Rate (MWR) on the metrics portfolio. The portfolio  in this plot uses an independent copula aggregation. Note that (1) the metrics portfolio successfully approximates the chatGPT evaluation, since the @P rankings largely agree with the @chatGPT rankings; (2) the R-SSD rankings outperform MWR baseline.}
        \label{fig:motivation}
    \vskip -0.15 in
\end{figure*}

A comprehensive evaluation  of  LLMs requires  addressing the following critical considerations:
\begin{enumerate}

\item \textbf{\emph{Interpretability.}} Evaluation of foundation models is multi-dimensional in nature  and multiple metrics benchmark the models on different socio-technical dimensions that probe the trustworthiness of their outputs and their adherence to shared values and ethics.  \emph{It is critical to establish an aggregate-level measure to facilitate the interpretation and effective communication of the evaluation results.} 
\item \textbf{\emph{Risk Aware Benchmarking.}}  In natural language (and other) applications, metrics quantify important guardrails such as model's toxicity, safety, or robustness. Therefore, a comprehensive evaluation framework must incorporate a risk aware benchmarking. This entails ranking models based on the assessment of failure modes and tail statistics\footnote{I.e. understanding and quantifying low-probability high-risk events.}, providing a nuanced understanding of potential pitfalls. 

\item \textbf{\emph{Statistical Significance.}} Evaluating machine learning models is intimately connected to statistical significance testing (SST), although this framework is still underutilized: \cite{dror2018hitchhiker} reports almost 50\% of ACL papers miss SST indicators. With the ever increasing parametric complexity of LLMs, obtaining a reliable SST in evaluating foundation models becomes ever more urgent.

\end{enumerate}

We propose in this paper an evaluation framework that offers a principled solution and an efficient implementation that addresses each of these challenges. Our main contributions are:
\begin{enumerate}
    \item \textbf{\emph{Interpretable Metrics-Portfolio (Section \ref{sec:framework}).}} Drawing inspiration from econometrics and mathematical finance, we define a metrics-portfolio for aggregating metrics. This portfolio uses the notion of \emph{copula} to  normalize and aggregate metrics, yielding a single interpretable number assessing each output of a LLM. A higher value of the portfolio is preferable. We illustrate in Figure \ref{fig:motivation} panels (a) and (b) summary statistics  of the metrics portfolio aggregating a total of $8$ automatic metrics computed using $5K$ samples from  the Mix-instruct dataset \citep{jiang2023llm}. In panel (c) we  show that model ranking based on our metrics-portfolio aligns with human evaluation proxies such as chatGPT (Please refer to Appendix \ref{app:chatGPT} for details of how chatGPT score is computed).
    
    \item \textbf{\emph{Risk Aware Benchmarking via  Second Order Stochastic Dominance (Section \ref{sec:SD}).}}
 Stochastic orders define partial orders on random variables and play a vital role in econometrics and mathematical finance for comparing and selecting portfolios. We propose using stochastic order to select LLMs based on their metrics-portfolios. A portfolio dominates in the First Order Stochastic Dominance (FSD) if it has higher quantiles for all percentiles. However, in Figure \ref{fig:motivation} (Panel (a)), the quantiles of the metrics-portfolio of an LLM don't provide a clear ordering. Instead, we propose the use of Second Stochastic Dominance (SSD), where a portfolio dominates if it has higher Tail Values at Risk (TVAR) for all percentiles (also known as Conditional Value at Risk). TVAR, illustrated in Figure \ref{fig:motivation} (Panel (b)), represents normalized integrated quantiles, assessing the risks of low values in the portfolio. Small TVAR corresponds to fat left tails in the distribution of the portfolio, identifying risky LLMs as those with the lowest TVAR. For example, Flan-T5 emerges as the riskiest model in our running example.

\item \textbf{\emph{Statistical Significance via Dominance Tests. (Section \ref{Sec:testing})}} Armed with these notions of stochastic dominance, we define statistics that benchmark the \emph{relative} dominance of a model's portfolio on another (R-FSD and R-SSD in Panel (c) in Figure \ref{fig:motivation}). We subject these statistics to an asymptotic analysis, proving central limit theorems that provide the foundation for hypothesis testing with false discovery rate control. We then perform stochastic dominance hypothesis testings between all pairs of models. Having adjusted the confidence level of these tests, we aggregate these pairwise rankings to a single rank via rank aggregation techniques such as the Borda Algorithm \citep{borda}.
The resulting ranks, depicted in Panel (c) of Figure \ref{fig:motivation}, highlight that the portfolio of automatic metrics (@P) leads to a similar ranking to chatGPT score (@chatGPT) for both first and second stochastic order. To underscore the importance of risk aware benchmarking, we present the ranking of the metrics-portfolio produced by the ubiquitous Min Win Rate (MWR) used in LLM benchmarks \citep{liang2022holistic}(last column in Panel (c)). Flan-T5 ranks close to last with all other orders, but ranks 6 with MWR. This highlights that the ubiquituous MWR used in LLM benchmarks is  risky for ranking LLMs as it does not take into account failure modes of the model, and we caution practitioners of its pitfalls.
\end{enumerate}

\section{Stochastic Dominance}\label{sec:SD}
We first review notions of stochastic dominance and their relation to downside risk measures and risk averse preference modeling. 
 We use the notation of the seminal paper of \cite{ogryczak2002dual}, and assume that the random variables are standardized so that larger  outcomes are preferable. Throughout this Section,  the reader can think of the random variable $X$  as a metric evaluating the performance of model $A$ on a specific test set. Likewise, $Y$ represents the evaluation of model $B$. We defer the definition of metrics portfolio  to Section \ref{sec:framework}. In a multi-metric evaluation, as explained in the introduction, $X$ and $Y$ represent portfolios of evaluations of model $A$  and $B$ respectively.

\subsection{First and Second order Dominance and  Mean-Risk Models }

\paragraph{First Order Stochastic Dominance} The First-order Stochastic Dominance (FSD)  between real-valued random variables uses the right-continuous cumulative distribution (CDF) as a performance function. Specifically, for a real random variable $X$, define the first performance function $F^{(1)}_{X}: \mathbb{R}\to [0,1]$ as the CDF:
$F^{(1)}_{X}(\eta)= \mathbb{P}(X\leq \eta), \forall \eta \in \mathbb{R}$. The FSD  of $X$ on $Y$ is defined as follows:
\begin{equation}
X  \underset{\text{FSD}}{\succcurlyeq} Y \iff F^{(1)}_{X}(\eta) \leq F^{(1)}_{Y}(\eta), \forall \eta \in \mathbb{R},
\end{equation}
this intuitively means that for all outcomes $\eta$, the probability of observing smaller outcomes than $\eta$ is lower for $X$ than $Y$. 
An equivalent definition can be expressed using the quantile $F^{(-1)}_{X}$ (See e.g \cite{ogryczak2002dual}):
\begin{assbox}
\begin{equation}
X  \underset{\text{FSD}}{\succcurlyeq} Y \iff F^{(-1)}_{X}(p) \geq F^{(-1)}_{Y}(p), \forall p  \in (0,1],
\label{eq:FSD}
\end{equation}
\end{assbox}
where $F^{(-1)}_{X}: [0,1]\to \overline{\mathbb{R}}$ is the left-continuous  inverse of $F^{(1)}_{X}$:
$F^{(-1)}_{X}(p)= \inf \{\eta : F^{(1)}_{X}(\eta) \geq p \} \text{ for } p \in (0,1].$
We focus on this definition as it is more computationally and notationally friendly since the quantile function is always supported on $[0,1]$.

\paragraph{Second Order Stochastic Dominance} The Second-order Stochastic Dominance (SSD) is defined via the second performance function  $F^{(2)}_{X}: \mathbb{R}\to [0,1]$ that measures the area under the CDF:
$
F^{(2)}_{X}(\eta)=\int_{-\infty}^{\eta} F^{(1)}_{X}(x)dx , \text{ for } x \in \mathbb{R},
$
yielding:
\begin{equation}
X\underset{\text{SSD}}{\succcurlyeq} Y \iff   F^{(2)}_{X}(\eta) \leq F^{(2)}_{Y}(\eta), \forall \eta \in \mathbb{R}.
 \end{equation}
 Note that FSD implies SSD, hence SSD is a finer notion of dominance. While FSD implies that $X$ is preferred to $Y$ by any utility-maximizing agent preferring larger outcomes\footnote{I.e. having an increasing utility function.}, \cite{ogryczak2002dual} showed that SSD implies that $X$ is preferred to $Y$ by any \emph{risk-averse} agent preferring larger outcomes.\footnote{I.e. having an increasing and \emph{concave} utility function.} 
 Similarly to FSD, SSD can be measured with quantile functions via introducing the second quantile function also known as \emph{integrated quantiles} $F^{(-2)}_{X}:(0,1] \to \overline{\mathbb{R}}$
 \begin{equation}
F^{(-2)}_{X}(p)=\int_{0}^{p} F^{(-1)}_{X}(t)dt , \text{ for } t \in (0,1].
\end{equation}
Similarly to the FSD case, an equivalent more computationally friendly definition can be expressed in terms of the second quantile function (a proof of this equivalence can be found in Theorem 3.2 in \cite{ogryczak2002dual}):
\begin{assbox}
\begin{equation}
X\underset{\text{SSD}}{\succcurlyeq} Y \iff   F^{(-2)}_{X}(p) \geq F^{(-2)}_{Y}(p), \forall p \in (0,1].
\label{eq:SSD}
 \end{equation}
 \end{assbox}
This equivalence is not straightforward and is due to Fenchel duality between $F^{(2)}$ and $F^{(-2)}$. Using $p=1$ we see that SSD implies $\mu_{X}\geq \mu_{Y}$, where $\mu_{X}$ and $\mu_{Y}$ are means of $X$ and $Y$. 

\begin{table*}[t!]
\centering
\resizebox{\textwidth}{!}{\begin{tabular}{l|l|l|c}
Name & Risk Measure  & $\alpha-$ consistency with SSD \\
\hline
 Standard deviation & $\sigma_{X}=\sqrt{\mathbb{E}(X-\mu_{X})^2}$ &  not consistent \\
Absolute semi deviation & $\delta_{X}=\mathbb{E}(\mu_{X}-X)_{+}$ & $1-$ consistent\\
Negative Tail Value at Risk & $-\mathrm{TVAR}_{X}(p)= - \frac{F^{(-2)}(p)}{p}$& $1-$ consistent for all $p\in(0,1]$\\
Mean absolute deviation from a quantile &$h_{X}(p)=\mu_{x}-\frac{F^{(-2)}_{X}(p)}{p}$ & $1-$ consistent for all $p \in (0,1]$\\
Gini Tail & $\Gamma_{X}=2\int_{0}^1(\mu_Xp-F^{(-2)}_{X}(p))dp$& $1-$ consistent \\
 \hline
\end{tabular}}
\caption{Risk models and their $\alpha-$consistency with SSD. }
\label{tab:RiskCMain}
\vskip -0.15in
\end{table*}
\textbf{Mean -- Risk Models (MRM)} As noted earlier SSD is linked to risk aware benchmarking via the second performance function $F^{(2)}(.)$ measuring expected shortfall, and the negative second quantile function $-F^{(-2)}(p)$ that is an assessment of expected losses given outcomes lower than the $p$-quantile. 

\begin{definition}[Mean -- Risk Models] A mean -- risk model of a random variable  $X$ consists of the pair $(\mu_{X},r_{X})$, where $\mu_{X}$ is the mean of $X$, and $r_{X}$ is a functional that measures the risk of the random outcome $X$. 
\end{definition}

The consistency of a mean -- risk model with SSD is defined as follows:

\begin{definition}[SSD consistency of Mean -- Risk Models] A mean -- risk model $(\mu_{X},r_{X})$ is $\alpha-$consistent with SSD, if for $\alpha >0$ the following is true:
\begin{equation}
X\underset{\text{SSD}}{\succcurlyeq} Y  \implies \mu_X -\alpha r_x \geq \mu_{Y}-\alpha r_{Y}.
\end{equation}
\end{definition}
\vskip -0.15 in


The ubiquitous mean -- risk model in machine learning is $(\mu_{X},\sigma_{X})$, where $\sigma_{X}$ is the standard deviation. Unfortunately this model is not consistent with the SSD and has several limitations as it implies Gaussianity of the outcomes or a quadratic utility function. 
We give in Table \ref{tab:RiskCMain}  risk measurements and their $\alpha-$consistency (proofs in \cite{ogryczak2002dual}). Note that  in contrast FSD is only consistent with the Mean-VaR risk model (Mean-Value at Risk) for all $p\in [0,1]$. VaR does not provide a refined tail assessment.  

\subsection{Relaxations of Stochastic Dominance}\label{sec:RelaxDominance}
\vskip -0.10in
Recalling the definitions of FSD and SSD in Equations \eqref{eq:FSD} and \eqref{eq:SSD}, in the finite-sample regime it is hard to test for these relations as one needs to show the infinite-sample quantile or second quantile properties hold uniformly over all $p \in (0,1]$. This difficulty motivated the relaxation of stochastic dominance to an almost stochastic dominance pioneered by \cite{leshno2002preferred}. These relaxations were revisited for the first order by  \cite{alvarez2014contamination} who later proposed an optimal transportation approach to assess almost first stochastic order \citep{del2018optimal}. 

\textbf{Almost FSD ($\varepsilon$-FSD)}  Following \cite{leshno2002preferred}, \cite{del2018optimal} relaxed FSD (Equation \eqref{eq:FSD}) via the violation ratio of FSD. $X  \underset{\varepsilon- \text{FSD}}{\succcurlyeq} Y $ if and only if: 

\begin{assbox}
\begin{equation}
\varepsilon_{\mathsf{W}_2}(F_{X},F_{Y})= \frac{\int_0^1(F_{Y}^{(-1)}(t)-F_{X}^{(-1)}(t))^2_+ dt}{\mathsf{W}^2_2(F_{X},F_{Y})} \leq \varepsilon,
\label{eq:epsFSD}
\end{equation}
\end{assbox}
 where $\mathsf{W}_2$ is the Wasserstein -2 distance between $F_X$ and $F_{Y}$.This ratio corresponds to a measure of the ``area" of violation of the FSD dominance of $X$ on $Y$. Note that $0\leq \varepsilon_{\mathsf{W}_2}(F_{X},F_{Y})\leq 1$, with value $0$ if  $X  \underset{\text{FSD}}{\succ} Y $ and $1$ if $Y  \underset{\text{FSD}}{\succ} X$. For $\varepsilon \in (0,\frac{1}{2}]$, 
Figure \ref{fig:epsilonFSD} in Appendix \ref{app:fig} illustrates $\varepsilon$-FSD, dashed areas represent the violation set. 

\textbf{Almost SSD ($\varepsilon$-SSD)} 
We define $\varepsilon$-SSD, for $\varepsilon \in(0,\frac{1}{2})$, by relaxing Equation \eqref{eq:SSD} as follows: $X  \underset{\varepsilon- \text{SSD}}{\succcurlyeq} Y  $ if and only if
\begin{assbox}
\begin{equation}
\varepsilon_{IQ}(F_{X},F_{Y}) = \frac{\int_0^1(F_{Y}^{(-2)}(t)-F_{X}^{(-2)}(t))^2_+ dt}{d_{IQ}^2(F_{X},F_{Y})} \leq \varepsilon,
\label{eq:eps-SSD}
\end{equation}
\end{assbox}
where $d_{IQ}$ is the $L_2$ distance between the Integrated Quantiles $(F^{(-2)})$. This ratio corresponds to a measure of the ``area" of violation of the SSD dominance of $X$ on $Y$.
Figure \ref{fig:epsSSD} in Appendix \ref{app:fig}  illustrates the second order, dashed areas represent the violation set of SSD of $X$ on $Y$.  Appendix \ref{app:SD} gives a more detailed account on almost stochastic dominance.

\subsection{Relative Stochastic Dominance}\label{sec:RelDom}
In the remainder of the paper, we refer to the FSD violation ratio as $\varepsilon_{\mathsf{W}_2}(F_{X},F_{Y}) \equiv \varepsilon^{(1)}(F_{X},F_{Y})$ and to the SSD violation ratio as $\varepsilon_{IQ}(F_{X},F_{Y}) \equiv \varepsilon^{(2)}(F_{X},F_{Y})$.
One of the shortcomings of almost stochastic dominance is the need to fix a threshold $\varepsilon$ on the violation ratio. When comparing two random variables, setting a threshold is a viable option. Nevertheless, when one needs to rank multiple variables $X_1,\dots, X_k$ (considering all pairwise comparisons), setting a single threshold that would lead to a consistent relative stochastic dominance among the $k$ variables becomes challenging. To alleviate this issue, we draw inspiration from relative similarity and dependence tests \citep{bounliphone2016test,bounliphone2016fast} that circumvent the need for a threshold via relative pairwise testings.

For $\ell \in \{1,2\}$ (i.e for FSD or SSD) we consider all pairs of violations ratios:
\[\varepsilon^{(\ell)}_{ij}=\varepsilon^{(\ell)}(F_{X_i},F_{X_j}) \text{ for } i, j \in \{1\dots k\},i\neq j,\]
noting that $\varepsilon^{(\ell)}_{ij}+\varepsilon^{(\ell)}_{ji}=1. $
Let $F=(F_{X_1},\dots F_{X_k})$.
We define the one-versus-all violation ratio of the dominance of $X_i$ on all other variables $X_j, j\neq i$ :
\[\varepsilon^{(\ell)}_i(F) = \frac{1}{k-1}\sum_{j\neq i }\varepsilon^{(\ell)}_{ij}. \]
We then define relative stochastic dominance for both orders, R-FSD an R-SSD respectively:
    \begin{align*}
   & X_{i_1} \underset{R-\text{FSD}}{\succcurlyeq} X_{i_2} \dots  \underset{R-\text{FSD}}{\succcurlyeq}  X_{i_{k}}\\
   & \iff  \varepsilon^{(1)}_{i_1}(F ) \leq \dots\leq \varepsilon^{(1)}_{i_k}(F )   
   & \text{ and, }\\
   &   X_{i_1} \underset{R-\text{SSD}}{\succcurlyeq} X_{i_2} \dots  \underset{R-\text{SSD}}{\succcurlyeq}  X_{i_{k}}\\
   &\iff  \varepsilon^{(2)}_{i_1}(F ) \leq \dots\leq \varepsilon^{(2)}_{i_k}(F )   
\end{align*}

In this definition of relative stochastic dominance, the most dominating model is the one with the lowest one-versus-all violation ratio and to test for relative dominance of $X_i$ on $X_j$ we can look at the following statistics:
\begin{equation}
   \Delta \varepsilon^{(\ell)}_{ij}(F)=\varepsilon^{(\ell)}_i(F) -  \varepsilon^{(\ell)}_j(F),
\end{equation}
and we have the following threshold-free test for relative order:\footnote{For comparing $k=2$ random variables, these $r$-FSD and $r$-SSD tests reduce to $0.5$-FSD and $0.5$-SSD absolute tests, respectively.}
\begin{assbox}
\begin{equation}
   X_{i} \underset{R-\text{FSD}}{\succcurlyeq} X_{j}  \iff  \Delta \varepsilon^{(1)}_{ij}(F) \leq 0 
   \label{eq:R-FSD}
\end{equation}
\begin{equation}
   X_{i} \underset{R-\text{SSD}}{\succcurlyeq} X_{j}  \iff  \Delta \varepsilon^{(2)}_{ij}(F) \leq 0 
   \label{eq:R-SSD}
\end{equation}
\end{assbox}

\section{Testing For Almost and Relative Stochastic Dominance}\label{Sec:testing}
Given empirical samples from $F_{X}$ and $F_{Y}$ we perform statistical testing of the almost and relative stochastic dominance of $X$ on $Y$ given empirical estimates of the statistics given in Sections \ref{sec:RelaxDominance} and \ref{sec:RelDom}. A key ingredient for quantifying the statistical significance of such tests is a central limit theorem that guarantees that the centered empirical statistics is asymptotically Gaussian at the limit of infinite sample size. Given $n$ samples  from $F_{X}$ ($m$ from $F_{Y}$ respectively), we denote $F^{n}_{X}$ and  $F^{m}_{Y}$ the corresponding empirical distributions.  For $\varepsilon_0-$ FSD, \cite{del2018optimal} studied the following hypothesis testing $H_0: X  \underset{\varepsilon_0- \text{SSD}}{\cancel{\succcurlyeq}} Y $ versus the alternative $H_a: X  \underset{\varepsilon_0- \text{SSD}}{{\succcurlyeq}} Y $. Using \eqref{eq:FSD}, this amounts to the following null hypothesis : 
$H_0: \varepsilon_{\mathsf{W}_2}(F^n_{X},F^m_{Y}) > \varepsilon_0.$ \cite{del2018optimal} showed the asymptotic normality of the empirical statistics:
\cite{del2018optimal,ulmer2022deep} propose to reject $H_0$ with a confidence level $1-\alpha$ if:
\begin{equation}
\varepsilon_{\mathsf{W}_2}(F^n_{X},F^m_{Y})  \leq \varepsilon_0 + \sqrt{\frac{m+n}{mn}}  \sigma^2(F_X, F_Y)\Phi^{-1}(\alpha),
\label{eq:CLT-E-FSD}
\end{equation}
where $\Phi^{-1}$ is the quantile function of a standard normal.

For the tests we propose below, we assume the following structure on the underlying CDFs to derive the corresponding central limit theorems (CLTs).

    \begin{assum}[Regularity]
        Let the CDF $F$ be supported on the interval $[-M,M]$ for some constant $M$, and have pdf $f$ such that $\frac{f'(p)}{f^3(p)}$ is bounded for almost every $p$ for which $f(p) > 0$ (i.e. all $p$ in the support of $f$). 
          \label{assumption}
    \end{assum}

\textbf{$\varepsilon$-SSD Testing} Similar to $\varepsilon$-FSD, using the definition in \eqref{eq:SSD} we propose to test using the  following null hypothesis for testing for $\varepsilon_0$-SSD:
\[H_0: \varepsilon_{IQ}(F^n_{X},F^{m}_{Y}) > \varepsilon_0 \]
Supposing Assumption \ref{assumption} holds for $F_X$, $F_Y$ and assuming $\frac{n}{n + m} \rightarrow \lambda$ for some $\lambda$, we state  a Central Limit Theorem for the second order statistics  (Theorem \ref{thm:CLT1Main}, proved in Appendix \ref{app:CLT}).

\begin{theorem}[Central Limit Theorem for $\varepsilon$-SSD]\label{thm:CLT1Main}
    Assume that $F_{X}$, $F_{Y}$ are supported on intervals\footnote{The interval for $F_X$ and for $F_Y$ need not coincide.} in $[-M,M]$, and have pdfs $f_x,f_y$ such that $\frac{f'_x(p)}{f^3_x(p)}$, $\frac{f_y'(p)}{f_y^3(p)}$ are bounded almost everywhere on the support of $f_x$ and $f_y$ respectively. Assume we have $n$ samples from $F_X$ and $m$ samples from $F_{Y}$, with $n, m \rightarrow \infty$ such that $\frac{n}{n+m} \rightarrow \lambda$ for some $\lambda$. Then 
$\sqrt{\frac{mn}{m+n}}\left(\varepsilon_{IQ}(F^n_{X},F^m_{Y}) - \varepsilon_{IQ}(F_{X},F_{Y})\right) 
\rightarrow \mathcal{N}(0,\sigma_\lambda^2(F_{X},F_{Y}))$
where $\sigma_\lambda^2(F_{X},F_{Y}) = \frac{1}{d^8_{IQ}(F_{X},F_{Y})}\left[(1-\lambda)\mathrm{Var}(v_X(U)) + \lambda \mathrm{Var}(v_Y(U)) \right],
$
for $U\sim \mathrm{Unif}[0,1]$, $v_Y(t) = 2  \left(\frac{1}{f_y(F_{Y}^{-1}(t))}\right)\left( \int_t^1 (F^{(-2)}_{X}(p)  - F^{(-2)}_{Y}(p) )_+ dp\right),$ and $v_X(t) = 2  \left(\frac{1}{f_x(F_{X}^{-1}(t))}\right)\left( \int_t^1 (F^{(-2)}_{X}(p)  - F^{(-2)}_{Y}(p) )_- dp\right).$
\end{theorem}

Similarly to \eqref{eq:CLT-E-FSD}, Theorem \ref{thm:CLT1Main} suggests to reject $H_0$ with a confidence $1-\alpha$ if :
\begin{assbox}
\begin{equation}
\varepsilon_{IQ}(F^n_{X},F^m_{Y})  \leq \varepsilon_0 + \sqrt{\frac{m+n}{mn}}  \sigma_\lambda^2(F_X, F_Y)\Phi^{-1}(\alpha),
\label{eq:CLT-E-SSD}
\end{equation}
\end{assbox}
where (for the same reasons as the FSD case) $\sigma^2_\lambda$ is given by the central limit theorem.

\textbf{Relative Stochastic Dominance Testing} We turn now to relative stochastic dominance that we introduced in \eqref{eq:R-FSD}  and \eqref{eq:R-SSD}   for first and second orders. Given $n$ samples from $k$ random variables $(X_1\dots X_k)$, let  $F=(F_1, \dots,F_k)$ be the  marginals of $X_i$ and $F_n=(F_{1n}, \dots,F_{kn})$ denote the empirical marginals. To test  for R-FSD (resp R-SSD) of $X_{i_1}$ on $X_{i_2}$ we propose to test the following null hypothesis:
\[ 
H_0: \Delta \varepsilon^{(\ell)}_{ij}(F_n) > 0 ,\ell=1 \text{ or } 2
\]

Assuming that each $F_i$ satisfies Assumption \ref{assumption}, we state in Appendix \ref{app:theory} a central limit theorem for the relative second order statistics (Theorem \ref{thm:relative} proved in in Appendix \ref{app:relative}). A similar result holds for the relative first order statistics that we omit for brevity.  
Theorem \ref{thm:relative} suggests to reject $H_0$ with a confidence $1-\alpha$ if:
\begin{assbox}
\begin{equation}
 \Delta \varepsilon^{(2)}_{i_1,i_2}(F_{n}) \leq \sqrt{\frac{1}{n}}  \sigma_{relative}^2(F_X, F_Y)\Phi^{-1}(\alpha)
 \end{equation}
 \end{assbox}
where $\sigma_{relative}^2(F_X, F_Y)$ is given by the central limit theorem (similar test exists for R-FSD).

\noindent \textbf{Bootstrapping Heuristic} While the CLT  above provides an asymptotic value for the variance, in practice (as in the ASO framework of \citep{ulmer2022deep}) we estimate the variance with a bootstrapping heuristic \citep{Efron1993}. This estimate is nonasymptotic and hence should often be more accurate than the asymptotic value. Proving the consistency of the bootstrap for functions of quantiles is generally nontrivial \citep{shao2012jackknife}, but recall that the stochastic ordering can be defined in terms of either quantiles or CDFs. In Appendix \ref{app:boot} we provide a bootstrap consistency proof for the absolute statistics based on the CDF, leaving the quantile based proof for future work. 

\noindent \textbf{Multi-Testing Algorithm} Algorithm \ref{alg:SOMT} given in Appendix \ref{app:alg} summarizes the multi-testing setup for both relative and almost (absolute) FSD and SSD. The main idea behind Algorithm \ref{alg:SOMT} is to turn multi-testing to  pairwise testings i.e testing for stochastic dominance between all pairs of models using relative (or absolute) FSD or SSD. In order to ensure that this multi-testing has a confidence level $1-\alpha$, we   correct the individual test's  confidence level   by  dividing $\alpha$ by the number of all pairs \citep{bonferroni1936teoria}.  Then in order to combine the pairwise rankings  to a single rank, we use a simple Borda count  \citep{borda} rank aggregation algorithm.

\section{Distributional Risk Aware Benchmarking of
Foundation Models } \label{sec:framework}

\textbf{Setup} In this section we consider the multi-metric evaluation setup of a foundation model $A: \mathcal{X}\to \mathcal{O}$, using  $N$ metrics $m_i: \mathcal{O} \to \mathbb{R},i=1\dots N$, where $m_i$ are real valued functions evaluated on a test set $D$.
Without  loss of generality, assume that each of the metrics are standardized such that higher values of $m_i$ correspond to more desirable model performance. We model observed values for each metric $m_i$ as a continuous random variable $M_i$ with unknown CDF $F_{M_i}$. For a model $A: \mathcal{X}\to \mathcal{O}$ and a data sample $X\sim D$, we describe the evaluation of model $A$ with $m_i$ with the following random variable $M_i$: $M_{i}| A,X := m_i(A(X)),~ X \sim D, i=1\dots N,$
where the randomness arises from the data sampling procedure $X\sim D$, and (if applicable) the stochasticity of the model $A$, for example if the model uses sampling.
\textbf{Metrics Portfolio Aggregation and Selection using Stochastic Dominance }
Let $\lambda =(\lambda_1,\dots, \lambda_N)$ be  a probability vector  that represents the importance of the $m_i$ metrics to the model's end user. Inspired by the portfolio optimization literature, we model the user return from a model as a \emph{portfolio of  metrics $m_i$ evaluated on a test set $D$}. Following \citep{copula-aggregation,belgodere2023auditing}, we define this portfolio  as an  Independent copula, which forms a weighted geometric mean of the CDFs:
\begin{equation}
 R_{A}(X) = \exp\left(\sum_{i=1}^N \lambda_i \log F_{M_i}\left(m_i(A(X))\right)\right) 
 \label{eq:portfolio}
 \end{equation}
 Note that \eqref{eq:portfolio} normalizes the metrics using the CDF of the metric $M_i$, eliminating the issue of differing dynamic ranges. This CDF should be formed by pooling together the evaluations on all samples and from all models being compared, to ensure that the various $R_A$ are comparable. The CDF normalization is monotonic and hence it preserves the order of each metrics and allow us to aggregate in the probability space the metrics using a simple weighted geometric mean. Computing $R_{A}(X)$ for all test samples $X$, we can therefore characterize the distribution of the  metric portfolio of the model $A$. To compare two models it is enough to compare their corresponding portfolios, specifically, Model $A$ is preferred  to Model B using $\varepsilon$- or R-SSD:
\begin{equation}
R_{A}(X) \underset{\varepsilon- \text{ or } R- \text{SSD}}{\succcurlyeq}  R_{B}(X).
\end{equation}
Similar tests can be performed for FSD. 

Note that the portfolio aggregation in \eqref{eq:portfolio} does not take into account the  dependencies and  correlations between the metrics. To alleviate this, we explore using also the empirical copula \cite{empricalcopula} as a means of aggregation of the metrics as follows
\begin{equation}
 R^c_{A}(X) =\hat{C}\left( F_{M_1}\left(m_1(A(X))\right), \dots F_{M_N}\left(m_N(A(X))\right)\right ), 
 \label{eq:portfolioCopula}
 \end{equation}
where $\hat{C}$ is the empirical copula . Given $N$ samples $X_{\ell}$, $\ell=1\dots n$, the empirical copula is given by
$\hat{C}(u_1,\dots u_n)=\frac{1}{n} \sum_{j=1}^n \Pi_{i=1}^N \mathbbm{1}_{{F}_{M_i}\left(m_i(A(X_j))\right)< u_i}$. The empirical copula can be understood as an average mean win rate (with an ``and" operation on all metrics), that is computed on the CDF transformed scores of each evaluated sample.  The main advantage of the independent copula (IC) in \eqref{eq:portfolio} versus the empirical copula (EC) in \eqref{eq:portfolioCopula} is its computational efficiency ($O(nN)$ for IC versus $O(n^2N)$ for EC). 

\textbf{Multiple Models Comparison}
Given $k$ models $A_{\ell},\ell=1\dots k $  and their evaluations  $m_{i}(A_{\ell}(X)),\:X\sim D,\: i=1\dots N$, we pool all model evaluations for a metric to estimate the CDF of each metric $F_{M_i}$ and construct a portfolio for each model $R_{A_{\ell}}(X)$. We use our Relative Stochastic Dominance testing introduced in Section \ref{Sec:testing} and in Algorithm \ref{alg:SOMT} to rank the models by their metrics portfolio in relative SSD or FSD with a confidence level $1-\alpha$.

\textbf{Per Metric Stochastic Dominance and Rank Aggregation} 
We also explore another approach for multi-testing, by considering the stochastic dominance of the models on per-metric basis. This amounts to computing $N$ relative stochastic orders for each $\mathcal{M}_i=(m_i(A_1(X)), \dots, m_i(A_{\ell}(X)))$, $i=1\dots N$. This amounts to producing via Algorithm \ref{alg:SOMT}  a relative ranking $\pi_i$ of the models based on $\mathcal{M}_i$. A single rank $\pi$ is then obtained via rank aggregation with uniform weighting on the per-metric rankings $\pi_i,i=1\dots N$. We use for rank aggregation the R package of \citep{Pihur2009RankAggregAR}. For more details on rank aggregation, the reader is referred to Appendix \ref{app:rankagg}.

\section{Experiments}\label{sec:exp}
\begin{figure*}[ht!]
    \centering
    \begin{subfigure}[t]{0.5\textwidth}
        \centering
        \includegraphics[height=2in]{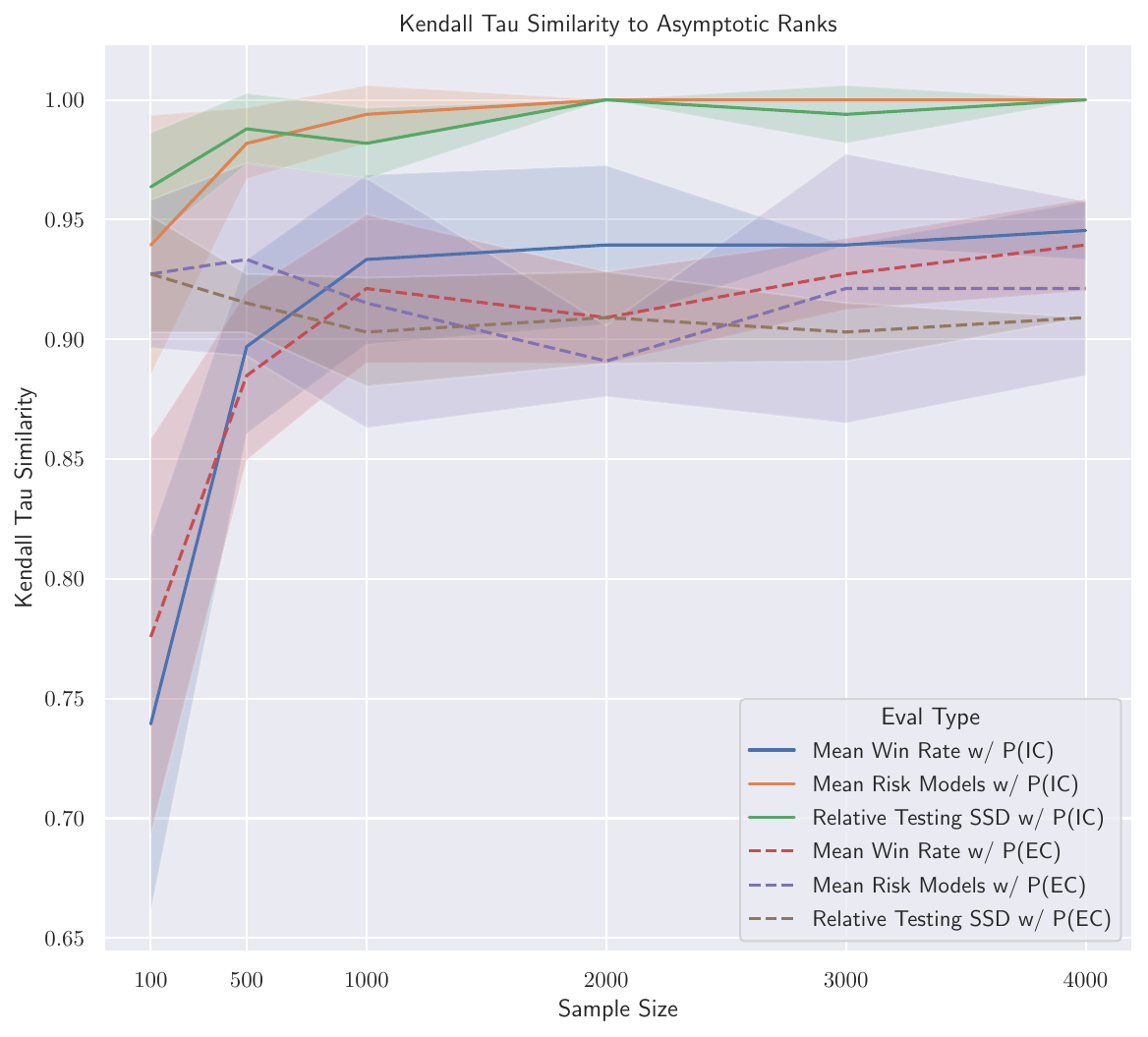}
        \caption{Asymptotic Rank Stability}
    \end{subfigure}%
    ~ 
    \begin{subfigure}[t]{0.5\textwidth}
        \centering
        \includegraphics[height=2in]{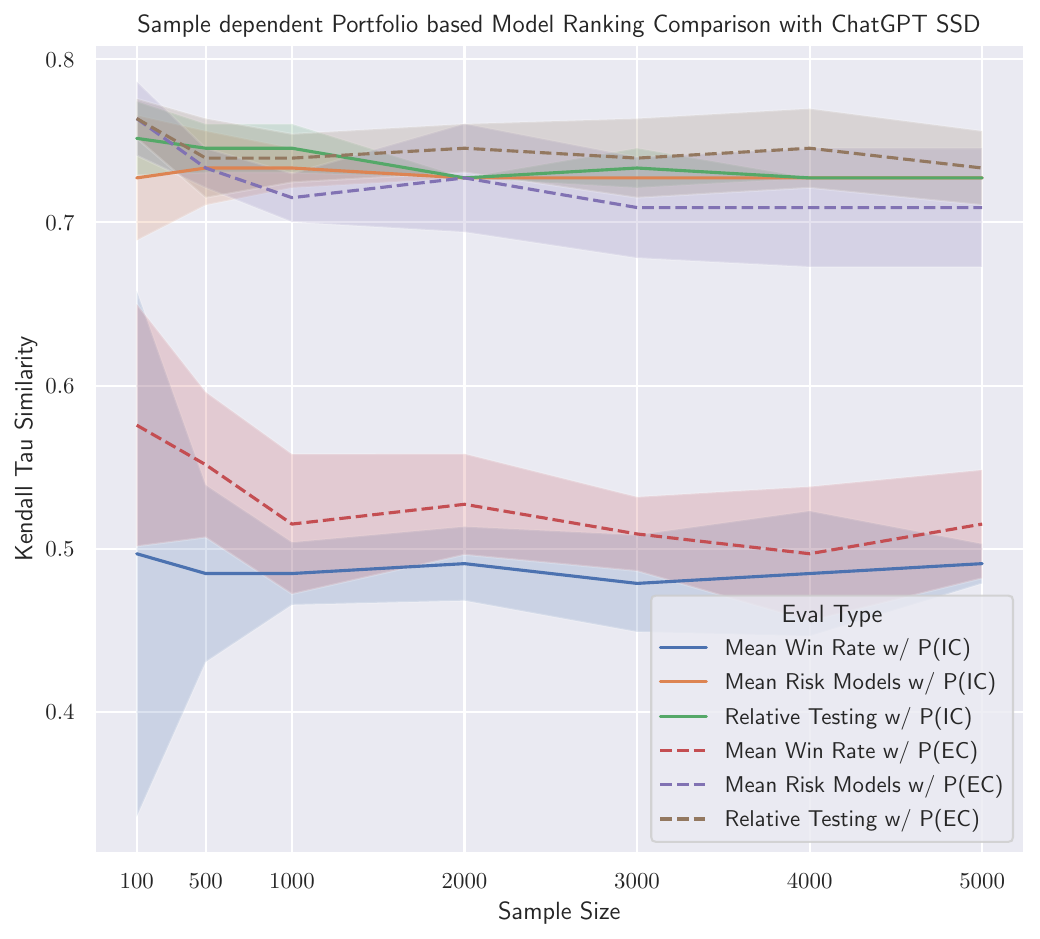}
        \caption{Rank Similarity to R-SSD @chatGPT Rank.}
    \end{subfigure}
    \caption{(a) On the Mix-instruct dataset, we compute the ranking resulting from each ranking method using varying sample sizes from 100 to 5K. We repeat each experiment 5 times. We report for each method,  the Kendall-Tau similarity between resulting ranks at each sample to the corresponding asymptotic rank at 5K samples.  We see that Relative SSD on independent copula portfolio P(IC) is  more stable in sample size than rank aggregation of all Mean Risk Models and more stable than MWR on the portfolio. The empirical dependent copula portfolio P(EC) does not have favorable asymptotics w.r.t to P(IC) since it suffers from the curse of dimension. (b) We use the same setup as in (a) but instead of Kendall-Tau similarity to the asymptotic rank of each method, we plot the similarity to R-SSD @ChatGPT rank at 5K samples. We see that MWR is inconsistent with chatGPT rank  while both R-SSD @P(IC) and (EC) and RA(MRM @P(IC)) have a Kendall-Tau similarity between 0.7 and 0.75. Interestingly, the dependent copula (EC) captures better chatGPT rank than independent copula (IC), hinting at the favorable role of the metric dependencies. }
    \label{fig:asymp}
\end{figure*}
\subsection{Validation of Statistical Significance}
We examine the statistical properties of our tests as a function of sample size.
We purposely design synthetic score distributions to represent challenging problems comprising large overlap between the distributions and considerable violation ratio, but where one would still like to have an ordering among the variables.
For this we consider the two Gaussian variables  $X\sim \mathcal{N}(0,1)$ and  and $Y \sim \mathcal{N}(0.5, 2)$. Figure \ref{fig:SyntTest}  in Appendix \ref{app:stsGaussian} shows that our tests have desirable statistical properties. We perform synthetic experiment on fat tailed distribution such as log normal (Fig. \ref{fig:SyntTest2} App. \ref{app:stsGaussian}).

\subsection{LLM evaluation with Stochastic Dominance }


We showcase LLM evaluation with stochastic dominance to benchmark two risks: drifting from instructions and outputting toxic content. The following datasets correspond to each risk we benchmark.

\textbf{Mix-Instruct Evaluation Data}  We use the data from \citep{jiang2023llm}, that consists of an instruction, an input sentence and an expected output from the user, as well as the output of a set of different LLMs. The dataset consists of a training set of 100K samples and a test set of 5K samples. \citep{jiang2023llm} used automatic metrics such as BARTscore and BLEU score comparing the LLM generation to the expected output in order to evaluate if each LLM followed the instruction. \citep{jiang2023llm} used also chatGPT to evaluate the generations (See Appendix \ref{app:chatGPT} for ChatGPT evaluation).  The  number of automatic metrics $N$ is 8, the total number of evaluated models $k$ is 12. Metrics are unified so that larger values are preferred.  \\

\noindent\textbf{Toxicity Evaluation} We use the real toxicity prompts dataset of \citet{Gehman2020RealToxicityPromptsEN}, and generate prompts completions  from the Llama 2 7b , Llama 2 13b,  Llama 2 70b , MosaicML MPT 30b and Tiiuae Falcon 40b models available in Opensource ($k=5$ models). We select two sets of prompts: toxic prompts (toxicity $> 0.8$, that gives $\sim  \!\! 10$K prompts ) and non-toxic prompts (toxicity $<0.2$, from which we randomly sample 10K). We sample from each model, 10 completions per prompt using nucleus sampling (top-$p$ sampling with $p=0.9$ and a temperature of 1). This procedure yields a dataset of $\sim \! \!200$K sentence completions per model. We evaluate the toxicity of these generations using the Perspective API, on the following toxicity metrics ($N=6$ metrics): Toxicity, Severe toxicity, Identity Attack, Insult, Profanity and Threat. Following \citet{liang2022holistic}, we evaluate the toxicity of generated completions only and refer to this as \textbf{\emph{Gen Only}} evaluation. In order to also give the context of the completion, we prepend the model generation with the prompt and evaluate the full sentence using Perspective API. We refer to this as \textbf{\emph{Prompt+Gen}}. The polarity of all toxicity metrics is unified so that high values refer to non toxic content (we use $-\mathrm{log}$ probabilities of Perspective API outputs).

\textbf{Evaluation Protocol and Baselines} We evaluate each of the use cases (instruction following  and toxicity) using the following absolute stochastic dominance tests: (1) $\varepsilon$-FSD (corresponds to the ASO evaluation of \cite{ulmer2022deep}) for $\varepsilon=0.08,0,25,0.4$.
(2) our proposed $\varepsilon$-SSD using the same values for $\varepsilon$, (3) our relative stochastic dominance R-FSD and R-SSD tests, (4) the Mean -- Risk models described in Table \ref{tab:RiskC}, and (5) the ranking produced by the Mean Win Rate (\textbf{MWR}) used by LLM leaderboards
such as HELM \citep{liang2022holistic}. As noted in Section \ref{sec:framework}, we either perform these tests on a \emph{metrics portfolio}  -- we refer to this as \textbf{test @P(IC)} when using the independent copula given in Equation \eqref{eq:portfolio} and  \textbf{test @P(EC)} when using the empirical copula given in Equation \eqref{eq:portfolioCopula} ; or  on a per metric basis leading to $N$ rankings of the models that we reduce to a single ranking via Rank Aggregation (RA) \citep{Pihur2009RankAggregAR} -- we refer to this as \textbf{RA(test @ M)}. In this naming convention, $\mathbf{test}$ takes values in $\{$MWR, $\varepsilon$-FSD, $\varepsilon$-SSD, R-FSD, R-SSD, Mean -- Risk Model ($\mu_{X}-r_{X}) \}$ where $r_X$ is a chosen risk from Table \ref{tab:RiskC}. We perform all our statistical tests with a significance  level $\alpha=0.05$, and use $1000$ bootstrap iterations.

\textbf{Efficient Implementation} We compare the computational complexity of our implementation for computing all stochastic orders to that of the \texttt{Deep-Significance} package \citep{deepsig} which implements $\varepsilon$-FSD in the ASO framework \citep{ulmer2022deep}, on the task of comparing models on the Mix-Instruct dataset (sample size 5K, $k=12$ models). Using the \texttt{Deep-Significance} implementation of \textsc{multi-ASO} in  \citep{ulmer2022deep} for $\varepsilon=0.25$ with just 3 bootstrap iterations\footnote{Limited to 3 for computational reasons.}, the test completes in 15min50s (averaged over 7 runs). Our code for relative and absolute testing performs all tests at once and relies on caching vectorization and multi-threading of the operations. Our code completes all tests in an average of just 17.7 s with 1000 bootstraps. Experiments were run on a CPU machine with 128 AMD cores, of which 2 were used.

\textbf{Mix-Instruct Results and Analysis} 
In Figure \ref{fig:asymp} we depict the asymptotics of the ranks resulting from our tests as function of the sample size. In  Figure \ref{fig:asymp} (a), we see that R-SSD  with  the portfolio aggregation with Independent Copula P(IC) has favorable asymptotics compared to R-SSD with dependent Empirical Copula P(EC). Indeed the empirical copula estimation suffers from the curse of dimension. On the other hand, we see in Figure \ref{fig:asymp} (b) that R-SSD with P(EC) captures better than P(IC) the ranks resulting from R-SSD with ChatGPT score. In other words, the dependent copula agrees more with the human evaluation proxy that is chatGPT. Note that the EC is expensive to compute and requires on average 1.5 hours  on 5K samples, whereas IC requires only 0.87 seconds. 

When compared with Mean Win Rate (MWR)  used in LLM leaderboards such as HELM \citep{liang2022holistic}, we see that it does not have good asymptotics nor agree with ChatGPT rankings, regardless of the aggregation technique used. This is due to the fact that MWR only counts wins and does not take into account how fat is the left tail of the distribution of the metric being benchmarked, possibly leading to overevaluation of risky models.

 Remarkably, the R-SSD ordering agrees with the rank aggregation of all (consistent) mean -- risk models, confirming the theoretical  link between second order dominance and risk averse decision making. The dependent copula EC with R-SSD leads to a better agreement with chatGPT R-SSD ranking than MRM models. Finally Tables  \ref{tab:tabMixInstructMain}  and Table  \ref{tab:mixinstapp} in Appendix \ref{app:AddExp} give additional results on R-FSD and the rank aggregation of all metrics, and how it compares to $\varepsilon-$ FSD and SSD.  

\begin{table*}[ht!]
\centering
\resizebox{\textwidth}{!}{\begin{tabular}{lccccc}

Scenario &  Llama 2 7b & Llama 2 13b & Llama 2 70b & MosaicML MPT 30b & Tiiuae Falcon 40b \\
\Xhline{5\arrayrulewidth}
&  &  &  &  &  \\
\textbf{All Combined (Toxic + Non-Toxic Prompts) } &  &  &  &  &  \\
&  &  &  &  &  \\
RA(R-FSD @M) (Gen Only) & \ranktwo{2} & \rankthree{3} & 5 & \rankone{1} & 4 \\
R-FSD @P(IC) (Gen Only)& \ranktwo{2} & \rankthree{3} & 5 & \rankone{1} & 4 \\
RA(R-SSD @M) (Gen Only) & \ranktwo{2} & \rankthree{3} & 5 & \rankone{1} & 4 \\
R-SSD @P(IC) (Gen Only) & \ranktwo{2} & \rankthree{3} & 5 & \rankone{1} & 4 \\

\hline
 &  &  &  &  &  \\
 
RA(R-FSD @M) (Prompt + Gen) & \rankthree{3} & 4 & 5 & \rankone{1} & \ranktwo{2} \\

RA(R-FSD @M) (Prompt + Gen)& \rankthree{3} & 4 & 5 & \rankone{1} & \ranktwo{2} \\

 R-SSD @P(IC) (Prompt + Gen)& \rankthree{3} & 4 & 5 & \rankone{1} & \ranktwo{2} \\
 R-SSD @P(IC) (Prompt + Gen) & \rankthree{3} & 4 & 5 & \rankone{1} & \ranktwo{2} \\

\Xhline{5\arrayrulewidth}



\end{tabular}}
\caption{ Toxicity Ranking using an Independent Copula portfolio aggregation of Perspective API metrics.} 
\label{tab:toxicityMain}
\end{table*}
\noindent \textbf{Toxicity Results and Analysis} 
Table \ref{tab:toxicityMain} shows the results of our tests on the combined set of toxic and non toxic prompts. Ablation studies on individual sets are given in Table \ref{tab:toxicity} in Appendix \ref{app:ECvsIC}. We make a few observations: First, overall the portfolio with independent copula approach agrees well with the rank aggregation of per-metric rankings. The portfolio is more computationally efficient as it needs to run the stochastic dominance test only on the portfolio, rather than running $N$ tests and aggregating them via rank aggregation.  An ablation study on empirical copula in Appendix \ref{app:AddExp} shows that it leads to a similar ranking as  the Independent Copula.
Secondly, on this dataset the R-FSD and R-SSD agree, with a few exceptions. Interestingly, when comparing models on model generation only, on toxic prompts MosaicML MPT stands out, while on non toxic prompts  Llama2 7B stands out and on the combined set Mosaic ML MPT stands out. On the combined set, we see for the llama family  that increased model size increases the toxicity of generations. This is in line with findings in the recent TrustLLM benchmark \cite{sun2024trustllm}.

\section{Conclusion}

In this paper we introduced  a distributional framework  for risk aware benchmarking and comparison of foundation models based on  multi-metric evaluations. Our framework has potential beyond the current applications presented here, being applicable wherever statistical significance while ranking assets for decision making is needed. We believe our tools for training models to be risk averse can be of significant use to practitioners and serve as a stepping stone towards solving the AI alignment problem.   

\flushcolsend
\section*{Impact Statement}
This paper presents a risk aware framework for benchmarking LLMs. In benchmarking LLM the stochastic nature of their generation and in presence of multiple metrics to be evaluated, our work offers a solution that gives raise to 1)  a sound aggregation of the metrics via the copula method 2) a risk aware evaluation that takes into account tail events of misalignment and not only the average behaviors thanks to the use of stochastic orders 3)  quantifies the uncertainty of the evaluation via statistical significance  testings. The potential societal
consequences of our work falls under AI governance as it allows a rigorous certification of compliance of LLMs with multitude of safeguards and dimensions. 

\bibliography{iclr2024_conference}
\bibliographystyle{icml2024}
\flushcolsend

\newpage
\appendix
\onecolumn

\begin{center}
    \large{\textbf{Supplementary Material}}
\end{center}

\section{Ablation Studies}\label{app:ablations}

\paragraph{Metrics Aggregation Versus Portfolio}

For portoflio,  computing ranking using  FSD and SSD  including the portfolio computation on $5K$ samples for $5$ bootstrap samples , we have mean execution time of  $32.01 \pm 4.51$ s. For FSd and SSD  ranking computation for all metrics, followed by rank  using pearson distance the   execution time is of  $254.99 \pm  16.76$ s. On the other hand, we observe on the mix-instruct dataset a consistency of ranks between these two approaches (FSD or SDD on portfolio \& FSD or SSD on all metrics followed by rank aggregation) as quantified by the kendall-tau similarity between the ranks: 
\begin{enumerate}
    \item Kendall Tau(R-SSD@P(IC), RA(R-SSD@M)) = 0.848
    \item  Kendall Tau(R-FSD@P(IC), RA(R-FSD@M)) = 0.878
    \item Kendall Tau(R-SSD@P(EC), RA(R-SSD@M)) = 0.848 
    \item Kendall Tau(R-FSD@P(EC), RA(R-SSD@M)) = 0.848
\end{enumerate}
    \vskip -0.1in

We see that these two approaches lead to similar ranks while portfolio approach leads to 7x speedups when using IC portfolios. 


           \vskip -0.2in


\section{Transforming Discrete Relative ChatGPT Scores to Absolute Real Valued Scores }\label{app:chatGPT}

We follow \cite{jiang2023llm} in mapping discrete chatGPT scores to real valued ones. Note that chatGPT scores for comparing models A and B are discrete and are one of these 4 options: A is better, B is better, Both are equally good, Both are equally bad. 

Given $m$ models we construct for each prompt sample $\ell=1\dots N$ a $m\times m$ comparison matrix with chatGPT:
\[ X_{\ell,ij} =  + 1 ,  X_{\ell,ji} =  - 1  \text{ if model $i$ is better} \]
\[ X_{\ell,ij} =  - 1 ,  X_{\ell,ji} =  + 1  \text{ if model $j$ is better} \]
\[ X_{\ell,ij} =  X_{\ell,ji} =  +0.5  \text{ if model $i$ and $j$ equally good  } \] 
\[ X_{\ell,ij} =  X_{\ell,ji} =  -0.5  \text{ if model $i$ and $j$ equally bad  } \] 

Then each model will define the following scalar score at each sample $\ell$:

\[s_{\ell,i}:= \sum_{j=1}^m (X_{\ell,ij} - X_{\ell,ji}). \]
hence we have a distribution of chatGPT score for each model :

\[p_{i}= \frac{1}{N}\sum_{i=1}^N \delta_{s_{\ell,i}}, i=1\dots m. \]

Note that the scores $s_{\ell,i}$ take on even integer values between $-2m$ and $2m$ inclusive, we treat the support as continuous and consider the following kernel density estimator with Gaussian kernel of width $\sigma$:
\[\hat{p}^{(\sigma)}_{i}(t)= \frac{1}{N}\sum_{i=1}^N \varphi\left(\frac{t - s_{\ell,i}}{\sigma}\right), t\in\mathbb{R}, \:i=1\dots m, \]
where $\varphi$ is the standard normal density.
In Figure \ref{fig:enter-label_2} below we plot chatGPT scores kernel density estimates for two models, openassistant and flan-t5:

\begin{figure}[ht!]
    \centering
    \includegraphics[scale=0.5]{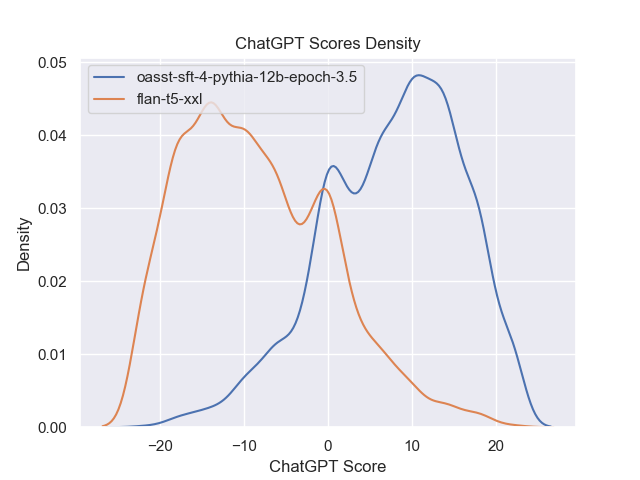}
    \caption{ChatGPT density scores for two models, open-assistant has clearly higher scores than the Flan-t5 models. }
    \label{fig:enter-label_2}
\end{figure}

\section{Multi-Testing Algorithm for Relative and Almost Stochastic Dominance}\label{app:alg}
Our multi-testing algorithm for relative and almost stochastic dominance is detailed in Algorithm \ref{alg:SOMT}.

In a nutshell our multi-testing consists of the following steps: 
\begin{enumerate}
\item For evaluation of each model compute summary statistics , i.e quantiles and integrated quantiles. 
\item For all pairs of models, compute statistics of absolute and relative tests by computing violation ratios. 
\item Compute the variance of these statistics via bootstrapping.
\item Perform the hypothesis testing for all pairs models with a corrected confidence level taking into account the number of all pairs 
\item Aggregate pairwise rankings to a rank using the Borda algorithm, that ranks the model by their number of wins in the stochastic dominance tests performed above. 

\end{enumerate}

\begin{algorithm}[htp!]
\caption{Stochastic Order Multi-testing (relative and \textcolor{violet}{\bf absolute})}
\label{alg:SOMT}
\begin{algorithmic}[1]
 \STATE {\bfseries Input:} $F_1, ...,F_{k}$, $k$ models we want to rank corresponding to empirical measure $p_1=\frac{1}{n}\sum_{i=1}^n \delta_{x^1_i}$, \dots $p_k=\frac{1}{n}\sum_{i=1}^n \delta_{x^k_i}$, \textbf{\textcolor{violet}{Threshold: $\tau$}.}
 \\
 \STATE {\bfseries Input:} Desired stochastic $\mathrm{order}\in\{1,2\}$, $B$ number of bootstraps, $m=K^2$ number of comparisons, significance level $\alpha$.
    \STATE {\bfseries \textcolor{blue}{ Cache the bootstraps samples and their statistics}}  
    
    \FOR{$j=1$ {\bfseries to} $k$}
    \STATE $p^{0}_j\gets p_j$
    \STATE {\bfseries Get Quantiles and Integrated Quantiles}
    \STATE $\mathrm{Q}_{0,j} \gets \textsc{GetQuantiles}(p_j)$
    \STATE $ \mathrm{IQ}_{0,j} \gets \textsc{GetIntegratedQuantiles}(p_j)$
    \FOR{$b=1$ {\bfseries to} $B$}
    
    \STATE {\bfseries Get Quantiles and Integrated Quantiles}
     \STATE $p^{b}_j \gets   \textsc{ResampleWithReplacement}(p_j, n) $ \COMMENT{using quantiles and uniform}
    \STATE $\mathrm{Q}_{b,j} \gets \textsc{GetQuantiles}(p^b_j)$
    \STATE $ \mathrm{IQ}_{b,j} \gets \textsc{GetIntegratedQuantiles}(p^b_j)$
   
      \ENDFOR
        \ENDFOR
 \STATE {\bfseries\textcolor{blue}{ Compute all violation ratios}}
     \STATE $\varepsilon_{b,i,j}\gets \textsc{ComputeViolationRatios}(F^b_i,F^b_j,\mathrm{order})$ for $b=0\dots B$,\quad$i,j=1\dots k, i\neq j$ \COMMENT{ratio of $\mathrm{Q}$ or $\mathrm{IQ}$ of $j> i$ by total area}
     \STATE $\varepsilon_{b,i,i}=0, \forall ~ b, i$

  \STATE {\bfseries \textcolor{blue}{Compute the sum statistics }}
   
\FOR{$b=0$ {\bfseries to} $B$}
\FOR{$i=1$ {\bfseries to} $k$}
\STATE $\varepsilon^i_{b} \gets\frac{1}{k-1} \sum_{j}\varepsilon_{b,i,j}$ 
\ENDFOR
\ENDFOR
  \STATE {\bfseries \textcolor{blue}{Compute the relative statistics }}
\STATE $\Delta \varepsilon^{i,j}_b= \varepsilon^i_{b}- \varepsilon^j_{b},\forall b, i,j$

\STATE {\bfseries \textcolor{blue}{Compute the Bootstrap Variance }}
\FOR{$i=1$ {\bfseries to} $k$}
\FOR{$j=1$ {\bfseries to} $k$}
\STATE $\sigma_{ij}=\sqrt{\frac{1}{B-1}\sum_{b=1}^B( \Delta \varepsilon^{i,j}_b - \textsc{Mean}(\Delta \varepsilon^{i,j}_b,b) )^2}$ 
\STATE \textbf{\textcolor{violet}{$\sigma^{\mathrm{abs}}_{ij}=\sqrt{\frac{1}{B-1}\sum_{b=1}^B(  \varepsilon_{b,i,j} - \textsc{Mean}( \varepsilon_{b,i,j},b) )^2}$}}
\ENDFOR
\ENDFOR

\STATE {\bfseries \textcolor{blue}{Compute the test  }}
\STATE $\mathrm{Win}_{ij}= \mathrm{Win}^{\mathrm{abs}}_{ij}=0$
\FOR{$i=1$ {\bfseries to} $k$}
\FOR{$j=1$ {\bfseries to} $k$}
  \IF{$i \neq j$ and $\Delta \varepsilon^{i,j}_0 - \frac{1}{\sqrt{n}}\sigma_{ij}\Phi^{-1}(
\alpha/k^2) \leq 0$} 
    \STATE $\mathrm{Win}_{ij}=1$ \COMMENT{with confidence level $1-\alpha/k^2$}
    \ENDIF
\textcolor{violet}{\IF{$i \neq j$ and $ \varepsilon_{0.i,j} - \frac{1}{\sqrt{n}}\sigma^{\mathrm{abs}}_{ij}\Phi^{-1}(
\alpha/k^2) \leq \tau$} 
    \STATE $\mathrm{Win}^{\mathrm{abs}}_{ij}=1$ \COMMENT{with confidence level $1-\alpha/k^2$}
\ENDIF}
\ENDFOR
\ENDFOR

$\mathrm{rank} =\textsc{Borda}(\mathrm{Win})$ \COMMENT{with confidence level $1-\alpha$}\\
\textcolor{violet}{$\mathrm{rank}_{\mathrm{abs}} =\textsc{Borda}(\mathrm{Win}^{\mathrm{abs}})$ \COMMENT{with confidence level $1-\alpha$}}
\STATE \textbf{Return} $\mathrm{rank}$, \textcolor{violet}{$\mathrm{rank}_{\mathrm{abs}}$}
\end{algorithmic}
\end{algorithm}

\begin{algorithm}
\caption{\textsc{ComputeViolationRatios}($F_a$,$F_b$,order)}
\begin{algorithmic}
\IF{order =1}
\STATE \textbf{Return}\ $\varepsilon_{\mathsf{W}_2}(F_a,F_b)$ in Definition \ref{def:RatioFSD}
\ELSIF{order=2}
\STATE \textbf{Return} $\varepsilon_{IQ}(F_a,F_b)$ in Definition \ref{def:RatioSSD}
\ENDIF
\end{algorithmic}
\end{algorithm}

\section{Absolute or Almost Stochastic Dominance} \label{app:SD}
\textbf{Almost FSD ($\varepsilon$-FSD)}  Following \cite{leshno2002preferred}, \cite{del2018optimal} relaxed FSD via the violation ratio of FSD:
\begin{assbox}
\begin{definition}[FSD Violation Ratio \citep{del2018optimal} ]  For $F_X\neq F_Y$ define the violation ratio:
\[\varepsilon_{\mathsf{W}_2}(F_{X},F_{Y})= \frac{\int_{A^{(1)}_0} (F_{X}^{(-1)}(t) - F_{Y}^{(-1)}(t))^2 dt}{ \int_{0}^1(F^{(-1)}_{X}(t) -F_{Y}^{(-1)}(t))^2dt} =\frac{\int_0^1(F_{Y}^{(-1)}(t)-F_{X}^{(-1)}(t))^2_+ dt}{\mathsf{W}^2_2(F_{X},F_{Y})},\]
where 
$A^{(1)}_0= \Big\{t \in (0,1):F_{Y}^{(-1)}(t)> F^{(-1)}_{X}(t) \Big \}$
is the violation set the relation $X  \underset{\text{FSD}}{\succcurlyeq} Y $, and $\mathsf{W}_2$ is the Wasserstein$-2$ distance.
\label{def:RatioFSD}
\end{definition}
\end{assbox}
Note that $0\leq \varepsilon_{\mathsf{W}_2}(F_{X},F_{Y})\leq 1$, with value $0$ if  $X  \underset{\text{FSD}}{\succ} Y $ and $1$ if $Y  \underset{\text{FSD}}{\succ} X$. For $\varepsilon \in (0,\frac{1}{2}]$, the relaxed FSD can be therefore defined as follows 
\begin{assbox}
\begin{equation}
X  \underset{\varepsilon- \text{FSD}}{\succcurlyeq} Y \iff \varepsilon_{\mathsf{W}_2}(F_{X},F_{Y}) \leq \varepsilon.
\end{equation}
\end{assbox}
Figure \ref{fig:epsilonFSD} in Appendix \ref{app:fig} illustrates $\varepsilon$-FSD, dashed areas represent the violation set. 

\textbf{Almost SSD ($\varepsilon$-SSD)} Note that the original definition of $\varepsilon$-FSD of $X$ on $Y$   in \cite{leshno2002preferred} is an $L_1$ definition and uses the CDF rather than quantiles: $\int_{-\infty}^{\infty}(F_{X}(x)-F_{Y}(x))_+ d x\leq \varepsilon \int_{\infty}^{\infty} |F_{X}(x) -F_{Y}(x) |dx .$  \cite{tzeng13} gave a similar $L_1$ definition for $\varepsilon$-SSD using the second performance function $F^{(2)}(.)$. According to \cite{tzeng13},  $X$ dominates $Y$ in the $\varepsilon$-SSD if  $\int_{-\infty}^{\infty}(F_{X}^{(2)}(x)-F_{Y}^{(2)}(x))_+ dt \leq \varepsilon \int_{-\infty}^{+\infty} |F^{(2)}_{X}(x) -F^{(2)}_{Y}(x)|dx .$ Following \cite{del2018optimal}, we redefine $\varepsilon$-SSD using second quantiles and with a $L_2$ definition, this eases the analysis  and practically the integration is on $(0,1]$ rather 
than $(-\infty, \infty)$.

We define the SSD violation ratio as follows:
\begin{assbox}
\vskip -0.05in
\begin{definition}[SSD Violation Ratio ]  For $F_X\neq F_Y$ define the violation ratio:
\[\varepsilon_{IQ}(F_{X},F_{Y})= \frac{\int_{A^{(2)}_0} (F_{X}^{(-2)}(t) - F_{Y}^{(-2)}(t))^2 dt}{ \int_{0}^1(F^{(-2)}_{X}(t) -F_{Y}^{(-2)}(t))^2dt} =\frac{\int_0^1(F_{Y}^{(-2)}(t)-F_{X}^{(-2)}(t))^2_+ dt}{d_{IQ}^2(F_{X},F_{Y})},\]
where 
$A^{(2)}_0= \Big\{t \in (0,1):F_{Y}^{(-2)}(t)> F^{(-2)}_{X}(t) \Big \}$
is the violation set the relation $X  \underset{\text{SSD}}{\succcurlyeq} Y $, and $d_{IQ}$ is the $L_2$ distance between the Integrated Quantiles $(F^{(-2)})$.
\label{def:RatioSSD}
\end{definition}
\end{assbox}

We are now ready to define $\varepsilon$-SSD, for $\varepsilon \in(0,\frac{1}{2})$:
\begin{assbox}
\begin{equation}
X  \underset{\varepsilon- \text{SSD}}{\succcurlyeq} Y \iff \varepsilon_{IQ}(F_{X},F_{Y}) \leq \varepsilon
\label{eq:epsSSD}
\end{equation}
\end{assbox}
Figure \ref{fig:epsSSD} in Appendix \ref{app:fig}  illustrates the second order, dashed areas represent the violation set of SSD of $X$ on $Y$.  Integrated  quantiles fully characterize one dimensional distributions as can be seen from the Theorem \ref{thm:metric} stated and  proved in Appendix \ref{app:metric}:

\section{Related Works on Stochastic Dominance}

\textbf{Stochastic Dominance} In \cite{dror2018hitchhiker,dror2019deep,ulmer2022deep,simpson2021statistical} a distributional assessment of the models based on stochastic dominance was proposed to overcome the limitations of the ubiquitous Mean-Variance Risk model used in machine learning. 

\citep{ulmer2022deep} used first order almost stochastic dominance and advocated for selecting a model $A$ over $B$ based on a metric $m_i$ if:
$M_{i}| A,X  \underset{\varepsilon-\text{FSD}}{\succcurlyeq} M_{i}| B,X.$ We expand this to the Relative-FSD. In natural language (and other) applications, it is often crucial to mitigate the risk of outputs with low metrics, especially when those metrics quantify important socio-technical guardrails such as model's toxicity, safety, or robustness.
 Unfortunately, the first stochastic ordering does not capture an assessment of the left tail behavior of  $M_{i}| A,X $ and $M_{i}| B,X $ and hence does not provide a risk-aware benchmarking \cite{ogryczak2002dual}.
 To alleviate this issue, we instead consider the \emph{second} order stochastic ordering  and use  our second order \emph{almost} or \emph{relative} stochastic dominance tests introduced in Section \ref{Sec:testing} for selecting a model A if:$M_{i}| A,X  \underset{\varepsilon \text{ or } R-\text{SSD}}{\succcurlyeq} M_{i}| B,X.$

\section{Supplement Discussions }
\subsection{Mean Risk Models}\label{app:MRM}
\begin{table}[ht!]
\centering
\resizebox{\textwidth}{!}{\begin{tabular}{l|l|l|c}
Name & Risk Measure  & $\alpha-$ consistency with SSD \\
\hline
 Standard deviation & $\sigma_{X}=\sqrt{\mathbb{E}(X-\mu_{X})^2}$ &  not consistent \\
Absolute semi deviation & $\delta_{X}=\mathbb{E}(\mu_{X}-X)_{+}$ & $1-$ consistent\\
Negative Tail Value at Risk & $-\mathrm{TVAR}_{X}(p)= - \frac{F^{(-2)}(p)}{p}$& $1-$ consistent for all $p\in(0,1]$\\
Mean absolute deviation from a quantile &$h_{X}(p)=\mu_{x}-\frac{F^{(-2)}_{X}(p)}{p}$ & $1-$ consistent for all $p \in (0,1]$\\
Gini Tail & $\Gamma_{X}=2\int_{0}^1(\mu_Xp-F^{(-2)}_{X}(p))dp$& $1-$ consistent \\
 \hline
\end{tabular}}
\caption{Risk models and their $\alpha-$consistency with SSD. }
\label{tab:RiskC}
\vskip -0.15in
\end{table}

Note that several risks in Table \ref{tab:RiskC} use the second quantile function as part of a benchmarking of the left tails of the outcomes. 
\subsection{$\delta-$ Consistency of Gini-Risk Models with $\varepsilon$-SSD}
\paragraph{ $\delta-$ Consistency of Gini-Risk Models with $\varepsilon$-SSD} We relax the definition of $\alpha-$ consistency of mean-risk models with SSD to $(\alpha,\delta)$ consistency with $\varepsilon$-SSD as follows:

\begin{definition}[$(\alpha,\delta)$ consistency of MRM with $\varepsilon$-SSD] A mean-risk model $(\mu_{X},r_{X})$ is $(\alpha,\delta)$ consistent with $\varepsilon$-SSD,  if there exists $\alpha,\delta >0$ such that $
X\underset{\varepsilon\text{-SSD}}{\succcurlyeq} Y  \implies \mu_X -\alpha r_x +\delta \geq \mu_{Y}-\alpha r_{Y} 
$
\end{definition}
It is easy to see that the Mean-Gini tail MRM of  $X$ and $Y$ is consistent with their $\varepsilon$-SSD:
\begin{proposition}
The Mean-Gini Tail MRM   is $(1,2\varepsilon^{\frac{1}{2}} d_{IQ}(F_{X},F_{Y}))$  consistent with $\varepsilon$-SSD.
\label{pro:deltacons}
\end{proposition}
\begin{proof} [Proof of Proposition \ref{pro:deltacons}]
\begin{align*}
\mu_X-\Gamma_{X}&=\mu_X- 2\int_{0}^1(\mu_X p-F^{(-2)}_{X}(p))dp= 2 \int_{0}^1 (F^{(-2)}_{X}(p) - F^{(-2)}_{Y}(p) + F^{(-2)}_{Y}(p)) dp\\
&= 2 \int_{0}^1 F^{(-2)}_{Y}(p)  + 2 \int_{A^{(2)}_0} (F^{(-2)}_{X}(p) - F^{(-2)}_{Y}(p)) dp +  2 \underbrace{\int_{[0,1]/A^{(2)}_0} (F^{(-2)}_{X}(p) - F^{(-2)}_{Y}(p)) dp}_{\geq 0}\\
&\geq 2 \int_{0}^1 F^{(-2)}_{Y}(p)  - 2 \int_{A^{(2)}_{0}} |F^{(-2)}_{X}(p) - F^{(-2)}_{Y}(p))| dp \\
&= \mu_Y-\Gamma_{Y} -  2\int_0^1 (F^{(-2)}_{Y}(p)- F^{(-2)}_{X}(p))_+dp\\
&\geq \mu_Y-\Gamma_{Y} - 2 \left(\int_{0}^1 dp\right)^{\frac{1}{2}} \left(\int_{0}^1 (F^{(-2)}_{Y}(p) - F^{(-2)}_{X}(p))^2_+ dp\right)^{\frac{1}{2}} 
(\text{Cauchy-Schwartz})\\
&\geq \mu_Y-\Gamma_{Y} -2 \varepsilon^{\frac{1}{2}}d_{IQ}(F_{X},F_{Y}) (\text{By assumption } X\underset{\varepsilon- \text{SSD}}{\succcurlyeq} Y  )
\end{align*}
\end{proof}

\subsection{Rank Aggregation}\label{app:rankagg}
Given $N$ ranks $\pi_{i},i=1\dots N$ represented as  permutations in $S_k$, the rank aggregation in \citep{Pihur2009RankAggregAR} solves the following problem :
\[ \min_{\pi \in S_k} \sum_{i=1}^N \alpha_i d(\pi,\pi_i),\]
where $\alpha_i\geq0, \sum_{i=1}^N \alpha_i=1$ represent importance of each ranking and $d$ is a distance between permutations. \citep{Pihur2009RankAggregAR} have multiple choices of distance such as Pearson or Kendall's-Tau. We fixed through out our experiments the distance to Pearson. 
 
 \subsection{Mean Win Rate and CDF normalizers in portfolio}
 
 To unpack the notations in \eqref{eq:portfolio}, consider a distribution $\mathcal{A}$ on models space. For a sample $X\sim D_i$ and a model $A \sim \mathcal{A}$, the metric $m_i()$ normalization through its CDF can be written as follows:
\begin{equation}
F_{M_i}(m_i(A(X))= \mathbb{E}_{B\sim \mathcal{A}} \mathbb{E}_{Y \sim D_i} \mathbbm{1}_{m_i(B(Y) \leq m_{i}(A(X))}.
\label{eq:CDFnorm}
\end{equation}
Hence for a model $A$ on each evaluated sample the CDF normalizer computes a soft ranking of  the evaluation of the model $A$ with a metric  $m_i$ on the sample $X$  with respect to all models and all samples.
\begin{remark}[Mean Win Rate ] Note that in LLM leaderborads such as HELM and Hugging face, the performance of a model $A$ evaluated with a metric $m_i$, is summarized via a Mean Win Rate (MWR) aggregated on models level looking on expected metrics
\begin{equation}
\mathrm{MWR}_{A,M_i} = \mathbb{E}_{B \sim \mathcal{A}}\mathbbm{1}_{\mathbb{E}_{X\sim D_i} \left[ m_i(B(X))\right] \leq \mathbb{E}_{X\sim D_i} \left[ m_i(A(X))\right]},
\label{eq:MWRagg}
\end{equation}
or aggregated on sample level marginalizing on models with a $\max$:
\begin{equation}
\overline{\mathrm{MWR}}_{A,M_i} = \mathbb{E}_{X\sim D_i} \mathbbm{1}_{ \max_{B\neq A }m_i(B(X)) \leq m_i(A(X))},
\label{eq:MWRsample}
\end{equation}
Contrasting \eqref{eq:CDFnorm} , \eqref{eq:MWRagg} and \eqref{eq:MWRsample} we see that instead of looking at the MWR summary statistics  that does not allow to consider all order statistics and relative ordering as well  the risks of  tails events, we consider a full distributional benchmarking in the metrics portfolio approach. 
\end{remark}

\section{Figures}\label{app:fig}
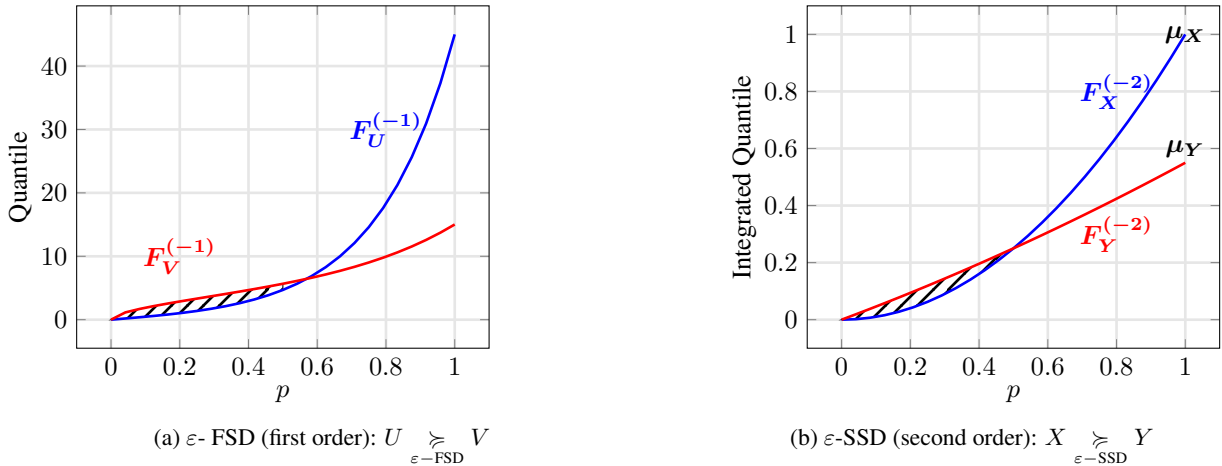
\begin{figure}[ht!]
\captionsetup[subfigure]{justification=centering}
\begin{subfigure}[b]{0.50\linewidth}
       
\begin{tikzpicture}[scale=0.80] 
    \begin{axis}[enlargelimits=0.1, grid=both,
    grid style={line width=1pt, draw=gray!20},xlabel={$p$ },
    ylabel={Quantile}]
    \addplot[name path=f,domain=0:1,blue,line width= 1pt] {sinh(4.5*x)};
    \addplot[name path=g,domain=0:1,red,line width= 1pt] {sinh(3*x)+5*sqrt(x) };

    \path[name path=axis] (axis cs:0,0) -- (axis cs:1,0);

    \addplot [draw,
            pattern=flexible hatch,
            hatch color=black,
            hatch distance=7pt,
            hatch thickness=1 pt,
        ] fill between [
            of=f and g,
            soft clip={domain=0:0.5},
        ];

    \node  at (axis cs:  0.8,  30) {\textcolor{blue}{$\bm{F^{(-1)}_{U}}$}};
    \node at (axis cs:  0.2,  10) {\textcolor{red}{$\bm{F^{(-1)}_{V}}$}};

    \end{axis}
\end{tikzpicture}  
\caption{$\varepsilon$- FSD (first order): $U  \underset{\varepsilon-\text{FSD}}{\succcurlyeq} V$ }
\label{fig:epsilonFSD}
\end{subfigure}
\hfill
\begin{subfigure}[b]{0.50\linewidth}
         \centering
\begin{tikzpicture} [scale=0.80]
    \begin{axis}[enlargelimits=0.1, grid=both,
    grid style={line width=1pt, draw=gray!20},xlabel={$p$ },
    ylabel={Integrated Quantile}]
    \addplot[name path=f,domain=0:1,blue,line width= 1pt] {x^2};
    \addplot[name path=g,domain=0:1,red,line width= 1pt] {0.1*x^2+0.45*x};

    \path[name path=axis] (axis cs:0,0) -- (axis cs:1,0);

    \addplot [draw,
            pattern=flexible hatch,
            hatch color=black,
            hatch distance=7pt,
            hatch thickness=1 pt,
        ] fill between [
            of=f and g,
            soft clip={domain=0:0.5},
        ];
   
    \node  at (axis cs:  1,  1) {$\bm{\mu_X}$};
    \node  at (axis cs:  1,  .6) {$\bm{\mu_{Y}}$};
    \node  at (axis cs:  0.8,  0.8) {\textcolor{blue}{$\bm{F^{(-2)}_{X}}$}};
    \node at (axis cs:  0.8,  0.3) {\textcolor{red}{$\bm{F^{(-2)}_{Y}}$}};

    \end{axis}
\end{tikzpicture}  
\caption{$\varepsilon$-SSD (second order): $X  \underset{\varepsilon-\text{SSD}}{\succcurlyeq} Y$ }
\label{fig:epsSSD}
\end{subfigure}
\caption{ \textbf{(a) An Example of Almost First Order Stochastic Dominance:} Plots of quantile functions of $U$ and $V$. Dashed areas is the violation set of first order stochastic dominance of $U$ on $V$. \textbf{(b) An Example of Almost Second Order Stochastic Dominance:} Plots of integrated quantile functions; dashed area is the violation set for the second order stochastic dominance of $X$ on $Y$. }
    \label{fig:enter-label}
    \vskip -0.25in
\end{figure}

\section{Central Limit Theorems} \label{app:theory}
\subsection{CLT for $\varepsilon$-SSD}
\begin{thmbox}
\begin{theorem}[Central Limit Theorem for $\varepsilon$-SSD]\label{thm:CLT1}
    Assume that $F_{X}$, $F_{Y}$ are supported on intervals\footnote{The interval for $F_X$ and for $F_Y$ need not coincide.} in $[-M,M]$, and have pdfs $f_x,f_y$ such that $\frac{f'_x(p)}{f^3_x(p)}$, $\frac{f_y'(p)}{f_y^3(p)}$ are bounded almost everywhere on the support of $f_x$ and $f_y$ respectively. Assume we have $n$ samples from $F_X$ and $m$ samples from $F_{Y}$, with $n, m \rightarrow \infty$ such that $\frac{n}{n+m} \rightarrow \lambda$ for some $\lambda$. Then 
    \begin{align*}
\sqrt{\frac{mn}{m+n}}\left(\varepsilon_{IQ}(F^n_{X},F^m_{Y}) - \varepsilon_{IQ}(F_{X},F_{Y})\right) 
&\rightarrow \mathcal{N}(0,\sigma_\lambda^2(F_{X},F_{Y}))
\end{align*}
where 
\[
\sigma_\lambda^2(F_{X},F_{Y}) = \frac{1}{d^8_{IQ}(F_{X},F_{Y})}\left[(1-\lambda)\mathrm{Var}(v_X(U)) + \lambda \mathrm{Var}(v_Y(U)) \right],
\]
for $U\sim \mathrm{Unif}[0,1]$, $v_Y(t) = 2  \left(\frac{1}{f_y(F_{Y}^{-1}(t))}\right)\left( \int_t^1 (F^{(-2)}_{X}(p)  - F^{(-2)}_{Y}(p) )_+ dp\right),$ and $v_X(t) = 2  \left(\frac{1}{f_x(F_{X}^{-1}(t))}\right)\left( \int_t^1 (F^{(-2)}_{X}(p)  - F^{(-2)}_{Y}(p) )_- dp\right).$
\end{theorem}
\end{thmbox}

\begin{remark}[Non-independent samples]\label{remark:dependent}
Theorem \ref{thm:CLT1} assumes that the $n$-sample from $F_{X}$ is independent of the $m$-sample for $F_{Y}$. Consider instead the setting where there are $n$ samples from $F_{X}$ and $F_{Y}$ that are dependent (e.g. $X$, $Y$ are evaluations of different models applied to the same data). We can describe general dependence structure as the following. Suppose $(X,Y)$ has marginals $X \sim F_{X}$, $Y \sim F_{Y}$, with some unknown dependence structure (optionally described by the copula $C_{XY}(u_x,u_y) = \mathrm{Pr}(F_{X}(X) \leq u_x, F_{Y}(Y) \leq u_y)$). Let
\[
(U_x, U_y) = (F_{X}(X), F_{Y}(Y)) \sim C_{XY}.
\]
Note that $U_x$ and $U_y$ have marginals equal to $\mathrm{Unif}([0,1])$, but $U_x$ and $U_y$ may be dependent. Hence the variances in each term of the decomposition \eqref{eq:dcmp} in the appendix cannot be added. Instead, one should modify the result of Theorem \ref{thm:CLT1} to use
\[
\bar{\sigma}_\lambda^2(F_{X},F_{Y}) = \frac{1}{d^8_{IQ}(F_{X},F_{Y})}\mathrm{Var}\left[v_X(U_x) + v_Y(U_y) \right].
\]
\end{remark}

\subsection{ CLT for Relative Statistics  }
 We focus here on presenting the Central Limit Theorem for SSD. The relative FSD has a similar form and we omit its statement here.  

\begin{thmbox}
\begin{theorem}[Central limit Theorem for Relative SSD]\label{thm:relative}
    Assume $F_{1n}, \dots, F_{kn}$ are available and independent, and each $F_i$ satisfies the conditions of Theorem \ref{thm:CLT1}. Then
\[
\sqrt{n}\left(\Delta \varepsilon^{(2)}_{i_1,i_2}(F_{n})- \Delta \varepsilon^{(2)}_{i_1,i_2}(F) \right) \rightarrow_w \mathcal{N}\left(0,\frac{1}{(k-1)^2}\sum_{i=1}^k \sigma_i^2(i_1,i_2)\right). 
\]
where
\[
\sigma_{i}^2(i_1,i_2) = \left\{\begin{array}{ll} \mathrm{Var}\left( \frac{2 v^{(1)-}_{i_1 i_2}(U_i)}{d^4_{IQ}(F_{i_1}, F_{i_2})} + \sum_{j \neq i_1, i_2} \frac{v^{(1)-}_{i_1 j}(U_i)}{d^4_{IQ}(F_{i_1}, F_{j})} \right) & i = i_1\\
\mathrm{Var}\left( \frac{2 v^{(2)+}_{i_1 i_2}(U_i)}{d^4_{IQ}(F_{i_1}, F_{i_2})} - \sum_{j \neq i_1, i_2} \frac{v^{(1)-}_{i_2 j}(U_i)}{d^4_{IQ}(F_{i_2}, F_{j})} \right) & i = i_2\\
\mathrm{Var}\left(\frac{v^{(2)+}_{i_1 j}(U_i)}{d_{IQ}^4(F_{i_1}, F_j)}-\frac{v^{(2)+}_{i_2 j}(U_i)}{d_{IQ}^4(F_{i_2}, F_j)}\right)& i \neq i_1, i_2\end{array}\right.
\]
for $U_i \sim \mathrm{Unif}([0,1])$ all independent, and $v^{(1),-}_{ij}(t) = 2  \left(\frac{d F_i^{-1}(t)}{dt}\right)\left( \int_t^1 (F_i^{(-2)}(p)  - F_j^{(-2)}(p))_- dp\right),v^{(2),+}_{ij}(t) = 2  \left(\frac{d F_j^{-1}(t)}{dt}\right)\left( \int_t^1( F_i^{(-2)}(p)  - F_j^{(-2)}(p))_+ dp\right).$
\end{theorem}
\end{thmbox}
\begin{remark}[Dependent samples]
    If the $F_{in}$ are dependent, a similar expression to that shown in Remark \ref{remark:dependent} for the absolute testing case also holds here. The statement is omitted.
\end{remark}

\section{Proof of Theorem \ref{thm:metric}}
\label{app:metric}

\begin{thmbox}
\begin{theorem}[$d_{IQ}$ is a metric]\label{thm:metric} $d_{IQ}$ is a metric on the space of univariate distributions with continuous CDF, moreover, it metrizes the weak topology.
\end{theorem}
\end{thmbox}
First, we show that $d_{IQ}(F,G) = 0$ if and only if $F = G$. The forward direction is obvious. For the reverse direction, if $d_{IQ}(F,G) = 0$, then $F^{(-2)}(t) = G^{(-2)}(t)$ a.e. By the continuity of integrated quantiles, this implies $F^{(-2)} = G^{(-2)}$ everywhere. Then, since $F^{(-1)}(t)$ is simply the derivative of $F^{(-2)}(t)$ with respect to $t$\footnote{This follows because $F^{-2}$ is the integral of the finite-valued quantile function $F^{-1}(t)$.}, $F^{(-1)} = G^{(-1)}$ everywhere by differentiating both sides of $F^{(-2)}(t) = G^{(-2)}(t)$. Hence $F=G$ since distributions are uniquely determined by their quantile functions.

The triangle inequality follows from the triangle inequality of the $L_2$ norm, since $\sqrt{\int_0^1 (F^{(-2)}(t) - G^{(-2)}(t))^2 dt} = \|F^{(-2)}(t) - G^{(-2)}(t)\|_{L_2([0,1])}$. Hence $d_{IQ}$ is a metric. By Theorem 10 in \cite{gushchin2017integrated}, we know that random variable $X_{(i)} \rightarrow_w X$ (with cdf $F_{(i)}$) if and only if $F_{(i)}^{(-2)}$ converges uniformly to $F^{(-2)}$. Hence $d_{IQ}$ must metrize weak convergence.

\section{Proofs of Central Limit Theorems}

\subsection{Absolute Testing: Proof of Theorem \ref{thm:CLT1}}\label{app:CLT}

Note that for $U_i$ and $V_i$ an $n$-sample and an $m$-sample respectively from $\mathrm{Unif}([0,1])$, we can get $X_i$, $Y_i$ as $X_i = F^{-1}(U_i)$, $Y_i = G^{-1}(V_i)$. Let $H_{n,1}$ and $H_{m,2}$ be the empirical d.f.s of the $U_i$ and $V_i$ respectively. We have
\[
F_n^{-1}(t) = F^{-1}(H_{n,1}^{-1}(t)),
\]
hence
\[
F_n^{(-2)}(t) = \int_0^t F_n^{-1}(p)dp = \int_0^t F^{-1}(H_{n,1}^{-1}(p)) dp. 
\]

We are interested in
\[\varepsilon_{IQ}(F_n,G_m)= \frac{\int_{A_0} (F_n^{(-2)}(t) - G_m^{(-2)}(t))^2 dt}{d^2_{IQ}(F_n,G_m)},\]
where 
\[A_0= \Big\{t \in (0,1):G_m^{(-2)}(t)> F_n^{(-2)}(t) \Big \}, \]is the violation set.

It is shown in \cite{gushchin2017integrated} (Theorem 10 therein) that integrated quantiles converge uniformly, i.e. $F_n^{(-2)}(t) \rightarrow F^{(-2)}(t)$ pointwise. As an immediate consequence, we have
\[
\varepsilon_{IQ}(F_n,G_m) \rightarrow_{a.s.} \varepsilon_{IQ}(F,G).
\]


We apply the following decomposition and bound the two terms separately:
\begin{equation}\label{eq:decomp}
\varepsilon_{IQ}(F_n,G_m) - \varepsilon_{IQ}(F,G) = (\varepsilon_{IQ}(F_n,G_m) - \varepsilon_{IQ}(F,G_m)) + (\varepsilon_{IQ}(F,G_m) - \varepsilon_{IQ}(F,G)).
\end{equation}

We derive asymptotic normality of these terms for $G_m$, the proof for $F_n$ is identical by symmetry. 

We introduce the statistics
\[
S_m = \int_0^1 (F^{(-2)}(t) - G_m^{(-2)}(t))^2 dt
\]
\[
S_m^+ = \int_0^1 (F^{(-2)}(t) - G_m^{(-2)}(t))_+^2 dt
\]
\[
S_m^- = \int_0^1 (F^{(-2)}(t) - G_m^{(-2)}(t))_-^2 dt
\]
The nonrandom $S, S^+, S^-$ are defined similarly with $G$ instead of $G_m$.

Next, set
\[
T_m = \sqrt{m} (S_m - S)
\]
\[
T_m^+ = \sqrt{m} (S_m^+ - S^+)
\]
\[
T_m^- = \sqrt{m} (S_m^- - S^-).
\]

\begin{theorem}\label{thm:1}
Assume that $G$ is supported on an interval that is a subset of $[-M,M]$, and has pdf $g$ such that $\frac{g'(p)}{g^3(p)}$ is bounded almost everywhere on the support of $g$. Then 
    \[
    T_m = \alpha_{m,2}(v)+ o_P(1)
    \]
    \[
    T_m^+ = \alpha_{m,2}(v^+)+ o_P(1)
    \]
    \[
    T_m^- = \alpha_{m,2}(v^-)+ o_P(1)
    \]
    where we define $\alpha_{m,2}(t) = \sqrt{m} (t - H_{m,1}^{-1}(t))$ and $\alpha_{m,2}(v) = \int_0^1 v(t) \alpha_{m,2}(t) dt$, and
\[
v(t) = 2  \left(\frac{1}{g(G^{-1}(t))}\right)\left( \int_t^1 F^{(-2)}(p)  - G^{(-2)}(p) dp\right).
\]
\[
v^+(t) = 2  \left(\frac{1}{g(G^{-1}(t))}\right)\left( \int_t^1 (F^{(-2)}(p)  - G^{(-2)}(p))_+ dp\right),
\]
\[
v^-(t) = 2  \left(\frac{1}{g(G^{-1}(t))}\right)\left( \int_t^1 (F^{(-2)}(p)  - G^{(-2)}(p))_- dp\right).
\]

\end{theorem}
\begin{proof}
    We begin with $T_m$. Note that\footnote{Convergence here is uniform convergence of the integrated quantiles.}
    \begin{align*}
    T_m &= \sqrt{m} (S_m - S) \\
    &= \sqrt{m} \int_0^1 (F^{(-2)}(t) - G_m^{(-2)}(t))^2 - (F^{(-2)}(t) - G^{(-2)}(t))^2 dt\\
    &=  \sqrt{m} \int_0^1 \left[2F^{(-2)}(t) - G_m^{(-2)}(t) - G^{(-2)}(t)\right] (G^{(-2)}(t) - G_m^{(-2)}(t)) dt\\
    &\rightarrow 2\sqrt{m} \int_0^1 \left[F^{(-2)}(t)  - G^{(-2)}(t)\right] (G^{(-2)}(t) - G_m^{(-2)}(t)) dt\\
    &= 2\sqrt{m} \int_0^1 \left[F^{(-2)}(t)  - G^{(-2)}(t)\right] \left[\int_0^t G^{(-1)}(p) - G^{(-1)}(H_{m,1}^{-1}(p)))dp\right] dt
    \end{align*}


Let us do integration by parts:
\begin{align*}
    &2\sqrt{m} \int_0^1 \left[F^{(-2)}(t)  - G^{(-2)}(t)\right] \left[\int_0^t G^{(-1)}(p) - G^{(-1)}(H_{m,1}^{-1}(p)))dp\right] dt = \\
    &= 2\sqrt{m}\left[ \left( \int_0^1 F^{(-2)}(t)  - G^{(-2)}(t) dt\right) \left[\int_0^1 G^{(-1)}(t) - G^{(-1)}(H_{m,1}^{-1}(t)))dt\right]\right. \\&-\left. \int_0^1 \left( \int_0^t F^{(-2)}(p)  - G^{(-2)}(p) dp\right) \left[ G^{(-1)}(t) - G^{(-1)}(H_{m,1}^{-1}(t)))\right]dt \right]\\
    &=2\sqrt{m} \int_0^1 \left( \int_t^1 F^{(-2)}(p)  - G^{(-2)}(p) dp\right) \left[ G^{(-1)}(t) - G^{(-1)}(H_{m,1}^{-1}(t)))\right]dt\\
    &= 2\sqrt{m} \int_0^1 \left(\frac{d G^{-1}(t)}{dt}\right)\left( \int_t^1 F^{(-2)}(p)  - G^{(-2)}(p) dp\right) (t - H_{m,1}^{-1}(t) ){dt} \\&+ O\left(\sqrt{m} \int_0^1 \int_t^1 F^{(-2)}(p)  - G^{(-2)}(p) dp)(t - H_{m,1}^{-1}(t) )^2dt\right)\\
    &= 2\sqrt{m} \int_0^1 \left(\frac{d G^{-1}(t)}{dt}\right)\left( \int_t^1 F^{(-2)}(p)  - G^{(-2)}(p) dp\right) (t - H_{m,1}^{-1}(t) ){dt} + o_P(1).
\end{align*}
In the penultimate step we have used a first-order Taylor series on $G^{-1}(t)$ via the assumption that $\frac{d^2 G^{-1}(t)}{dt^2} = -\frac{g'(G^{-1}(t))}{g^3(G^{-1}(t))}$ is bounded almost everywhere, and in the final step we have noted that 
\begin{align*}
\sqrt{m} \int_0^1 \left(\int_t^1 F^{(-2)}(p)  - G^{(-2)}(p) dp\right)(t - H_{m,1}^{-1}(t) )^2dt &\leq  2\sqrt{m} \int_0^1  (t - H_{m,1}^{-1}(t) )^2dt\\
&= o_P(1),
\end{align*}
since the support of $F$ and $G$ lie in $[-M, M]$ and $\int_0^1  (t - H_{m,1}^{-1}(t) )^2dt= O_p(1/m)$.



We then have
\[
T_m = \alpha_{m,2}(v)+ o_P(1),
\]
where $\alpha_{m,2}(t) = \sqrt{m} (t - H_{m,1}^{-1}(t))$, and $\alpha_{m,2}(v) = \int_0^1 v(t) \alpha_{m,2}(t) dt$ where
\[
v(t) = 2  \left(\frac{d G^{-1}(t)}{dt}\right)\left( \int_t^1 F^{(-2)}(p)  - G^{(-2)}(p) dp\right).
\]
Similarly,
\[
T_m^+ = \alpha_{m,2}(v^+)+ o_P(1),\quad T_m^- = \alpha_{m,2}(v^-)+ o_P(1)
\]
where
\[
v^+(t) = 2  \left(\frac{d G^{-1}(t)}{dt}\right)\left( \int_t^1 (F^{(-2)}(p)  - G^{(-2)}(p))_+ dp\right),
\]
\[
v^-(t) = 2  \left(\frac{d G^{-1}(t)}{dt}\right)\left( \int_t^1 (F^{(-2)}(p)  - G^{(-2)}(p))_- dp\right).
\]
\end{proof}

\begin{corollary}\label{Cor:1}
    Assume that $G$ is supported on an interval in $[-M,M]$, and has pdf $g$ such that $\frac{g'(p)}{g^3(p)}$ is bounded almost everywhere on the support of $g$. Then as $m \rightarrow \infty$
    \[
    \sqrt{m}(\epsilon_{IQ}(F, G_m) - \epsilon_{IQ}(F, G)) \rightarrow_w \mathcal{N}(0,\sigma^2)
    \]
    and if additionally $n \rightarrow \infty$
    \[
    \sqrt{m}(\epsilon_{IQ}(F_n, G_m) - \epsilon_{IQ}(F_n, G)) \rightarrow_w \mathcal{N}(0,\sigma^2),
    \]
    if for $U \sim \mathrm{Unif}([0,1])$
    \[
    \sigma^2 = \frac{\mathrm{Var}(v^+(U))}{d^8_{IQ}(F,G)}
    \]
    is finite.
\end{corollary}
\begin{proof}
    Note that by Theorem \ref{thm:1}
    \[
    \sqrt{m}(\epsilon_{IQ}(F, G_m) - \epsilon_{IQ}(F, G)) = \sqrt{m}\left(\frac{S_m^-}{S_m} -\frac{S^-}{S}\right) = \frac{\sqrt{m}}{S S_m}(T_m^- - T_m ) \rightarrow -\frac{\alpha_{m,2}(v^+)}{S^2}
    \]
    since $S_m \rightarrow S$ a.s. Recalling the definition of $\alpha_{m,2}$ yields asymptotic normality with zero mean as in \cite{del2018optimal}, and variance as calculated in the corollary statement.

    The case of $\sqrt{m}(\epsilon_{IQ}(F_n, G_m) - \epsilon_{IQ}(F_n, G))$ follows similarly since integrated quantiles weakly converge as $F_n \rightarrow F$.
\end{proof}

Continuing with the main proof, recalling \eqref{eq:decomp} and using Corollary \ref{Cor:1} along with the asymptotic independence of the two terms and the fact that $\frac{n}{n+m} \rightarrow \lambda$, we have
\begin{align}
\nonumber\sqrt{\frac{mn}{m+n}}&\left(\varepsilon_{IQ}(F_n,G_m) - \varepsilon_{IQ}(F,G)\right) \\&= \sqrt{(1-\lambda)n}(\varepsilon_{IQ}(F_n,G_m) - \varepsilon_{IQ}(F,G_m)) + \sqrt{\lambda n}(\varepsilon_{IQ}(F,G_m) - \varepsilon_{IQ}(F,G))\label{eq:dcmp}\\
&\rightarrow \mathcal{N}(0,\sigma_\lambda^2(F,G))\nonumber
\end{align}
where
\[
\sigma_\lambda^2(F,G) = \frac{1}{d^8_{IQ}(F,G)}\left[(1-\lambda)\mathrm{Var}(v_F(U)) + \lambda \mathrm{Var}(v_G(U)) \right].
\]
Here, we have defined
\[
v_G(t) = 2  \left(\frac{1}{g(G^{-1}(t))}\right)\left( \int_t^1 (F^{(-2)}(p)  - G^{(-2)}(p))_+ dp\right),
\]
and
\[
v_F(t) = 2  \left(\frac{1}{f(F^{-1}(t))}\right)\left( \int_t^1 (F^{(-2)}(p)  - G^{(-2)}(p))_- dp\right).
\]


\subsection{Relative Testing: Proof of Theorem \ref{thm:relative}}\label{app:relative}
Note that
\begin{align*}
    \Delta \varepsilon_{IQ}
^{i_1,i_2}(F)&= {\epsilon}_{IQ}^{i_1}(F) - {\epsilon}_{IQ}^{i_2}(F)\\
&= \frac{1}{k-1} \left[\sum_{j \neq i_1} \epsilon_{IQ}^{i_1j} - \sum_{j \neq i_2} \epsilon_{IQ}^{i_2j} \right]\\
&= \frac{1}{k-1} \left[2 \epsilon_{IQ}^{i_1 i_2} -1 + \sum_{j \neq i_1,i_2} (\epsilon_{IQ}^{i_1j} - \epsilon_{IQ}^{i_2 j}) \right].
\end{align*}
For compactness, let us introduce the differencing notation $\phi(\cdot)|_{F}^{F_n} = \phi(F_n) - \phi(F)$. We seek a CLT on 
\begin{align*}
\sqrt{n}(\Delta \widehat{\varepsilon_{IQ}
^{i_1,i_2}}(F_{n})- \Delta \varepsilon_{IQ}
^{i_1,i_2}(F) )
&= \frac{\sqrt{n}}{k-1}{\left.\left(2 \epsilon_{IQ}(\cdot, F_{i_2,n}) + \sum_{j \neq i_1,i_2} \epsilon_{IQ}(\cdot, F_{j,n})\right) \right|_{F_{i_1}}^{F_{i_1,n}}}\\
&+ \frac{\sqrt{n}}{k-1}{\left.\left(2 \epsilon_{IQ}(F_{i_1}, \cdot) - \sum_{j \neq i_1,i_2} \epsilon_{IQ}(\cdot, F_{j,n})\right) \right|_{F_{i_2}}^{F_{i_2,n}}}\\
&+ \frac{\sqrt{n}}{k-1}{\sum_{j \neq i_1,i_2} \left.\left(\epsilon_{IQ}(F_{i_1}, \cdot) - \epsilon_{IQ}(F_{i_2}, \cdot)\right)\right|_{F_{j}}^{F_{j,n}}}\\
&\hspace{-1.5in}\rightarrow_w \underbrace{\frac{\sqrt{n}}{k-1}\left.\left(2 \epsilon_{IQ}(\cdot, F_{i_2}) + \sum_{j \neq i_1,i_2} \epsilon_{IQ}(\cdot, F_{j})\right) \right|_{F_{i_1}}^{F_{i_1,n}}}_{I}+ \underbrace{\frac{\sqrt{n}}{k-1}\left.\left(2 \epsilon_{IQ}(F_{i_1}, \cdot) - \sum_{j \neq i_1,i_2} \epsilon_{IQ}(\cdot, F_{j})\right) \right|_{F_{i_2}}^{F_{i_2,n}}}_{II}\\
&+ \underbrace{\frac{\sqrt{n}}{k-1}\sum_{j \neq i_1,i_2} \left.\left(\epsilon_{IQ}(F_{i_1}, \cdot) - \epsilon_{IQ}(F_{i_2}, \cdot)\right)\right|_{F_{j}}^{F_{j,n}}}_{III}
\end{align*}
where we have used the uniform convergence of integrated quantiles. Note that $I$, $II$, and each term in the sum in $III$ are all independent. 

Define
\[
v^{(1)}_{ij}(t) = 2  \left(\frac{d F_i^{-1}(t)}{dt}\right)\left( \int_t^1 F_i^{(-2)}(p)  - F_j^{(-2)}(p) dp\right),
\]
\[
v^{(2)}_{ij}(t) = 2  \left(\frac{d F_j^{-1}(t)}{dt}\right)\left( \int_t^1 F_i^{(-2)}(p)  - F_j^{(-2)}(p) dp\right),
\]
and $v^{(1)+}_{ij}$, $v^{(2)+}_{ij}$ similarly.
Then by the proof of Corollary \ref{Cor:1}, each term in $III$ converges to
\begin{align*}
\frac{\sqrt{n}}{k-1}\left.\left(\epsilon_{IQ}(F_{i_1}, \cdot) - \epsilon_{IQ}(F_{i_2}, \cdot)\right)\right|_{F_{j}}^{F_{j,n}} &\rightarrow -\frac{\alpha_{m,j}(v_{i_1j}^{(2)+})}{(k-1)d_{IQ}^4(F_{i_1}, F_j)} + \frac{\alpha_{m,j}(v_{i_2j}^{(2)+})}{(k-1)d_{IQ}^4(F_{i_2}, F_j)}\\
&= \frac{1}{k-1}\alpha_{m,j}\left(-\frac{v^{(2)+}_{i_1 j}}{d_{IQ}^4(F_{i_1}, F_j)}+\frac{v^{(2)+}_{i_2 j}}{d_{IQ}^4(F_{i_2}, F_j)}\right)\\
&\rightarrow_w \mathcal{N}\left(0,\frac{1}{(k-1)^2}\sigma_{j}^2(i_1,i_2)\right),\quad \forall j \neq i_1, i_2.
\end{align*}
where
\[
\sigma_j^2(i_1,i_2) = \frac{1}{(k-1)^2}\mathrm{Var}\left(\frac{v^{(2)+}_{i_1 j}(U)}{d_{IQ}^4(F_{i_1}, F_j)}-\frac{v^{(2)+}_{i_2 j}(U)}{d_{IQ}^4(F_{i_2}, F_j)}\right),\quad \forall j \neq i_1, i_2,
\]
and $U\sim \mathrm{Unif}([0,1])$. Similarly for $I$ and $II$,
\begin{align*}
    I &\rightarrow_w \mathcal{N}\left(0,\frac{1}{(k-1)^2}\sigma_{i_1}^2(i_1,i_2)\right)\\
    II &\rightarrow_w \mathcal{N}\left(0,\frac{1}{(k-1)^2}\sigma_{i_2}^2(i_1,i_2)\right)\\
\end{align*}
where\footnote{This $U\sim \mathrm{Unif}([0,1])$ is drawn simply for this variance calculation and is not dependent on anything outside of this equation.}
\[
\sigma_{i_1}^2(i_1,i_2) = \mathrm{Var}\left( \frac{2 v^{(1)-}_{i_1 i_2}(U)}{d^4_{IQ}(F_{i_1}, F_{i_2})} + \sum_{j \neq i_1, i_2} \frac{v^{(1)-}_{i_1 j}(U)}{d^4_{IQ}(F_{i_1}, F_{j})} \right),
\]
\[
\sigma_{i_2}^2(i_1,i_2) = \mathrm{Var}\left( \frac{2 v^{(2)+}_{i_1 i_2}(U)}{d^4_{IQ}(F_{i_1}, F_{i_2})} - \sum_{j \neq i_1, i_2} \frac{v^{(1)-}_{i_2 j}(U)}{d^4_{IQ}(F_{i_2}, F_{j})} \right).
\]
Putting everything together via independence,
\[
\sqrt{n}\left(\Delta \widehat{\varepsilon_{IQ}
^{i_1,i_2}}(F_{n})- \Delta \varepsilon_{IQ}
^{i_1,i_2}(F) \right) \rightarrow_w \mathcal{N}\left(0,\frac{1}{(k-1)^2}\sum_{i=1}^k \sigma_i^2(i_1,i_2)\right). 
\]




\section{Consistency of Bootstrapping}
\label{app:boot}
In this section, we consider the relaxation measure using the CDFs\footnote{The result using quantiles as described in the main text is less straightforward and if left for future work.}:
\[
\tilde{\epsilon}_{\ell}(F_X, F_Y) = \frac{\int_{-\infty}^{\infty} (F^{(\ell)}_Y (t) - F^{(\ell)}_X (t))_+^2 dt}{\int_{-\infty}^{\infty} (F^{(\ell)}_Y (t) - F^{(\ell)}_X (t))^2 dt}.
\]
Note that we can relax FSD as follows:
\begin{equation}
Y  \underset{\varepsilon- \text{FSD}}{\succcurlyeq} X \iff \tilde{\epsilon}_{1}(F_{X},F_{Y}) \leq \varepsilon.
\label{eq:epsFSDCDF}
\end{equation}
Similarly we can relax SSD as follows:
\begin{equation}
Y  \underset{\varepsilon- \text{SSD}}{\succcurlyeq} X \iff \tilde{\epsilon}_{2}(F_{X},F_{Y}) \leq \varepsilon.
\label{eq:epsSSDCDF}
\end{equation}

We will prove bootstrap consistency for $\ell = 1$ (approximate first order dominance), the proof for $\ell = 2$ (approximate second order dominance) is similar.

We seek to show that the bootstrapped variance $\mathsf{Var}(\tilde{\epsilon}_{1}(F^{n\ast}_X, F^{m\ast}_Y))$ is an asymptotically consistent estimator of $\mathsf{Var}(\tilde{\epsilon}_{1}(F^{n}_X, F^{m}_Y)$, i.e. their ratio goes to 1:
\[
\frac{\mathsf{Var}(\tilde{\epsilon}_{1}(F^{n\ast}_X, F^{m\ast}_Y))}{\mathsf{Var}(\tilde{\epsilon}_{1}(F^{n}_X, F^{m}_Y))} \rightarrow_p 1.
\]

Note we can write this as
\[
\frac{\mathsf{Var}(\tilde{\epsilon}_{1}(F^{n\ast}_X, F^{m\ast}_Y)}{\mathsf{Var}(\tilde{\epsilon}_{1}(F^{n}_X, F^{m}_Y))}  \rightarrow_{p} \frac{\mathsf{Var}(T(F^{n\ast}_X, F^{m\ast}_Y))}{\mathsf{Var}(T(F^n_X, F^m_Y))},
\]
where
\[
T(F_X, F_Y) = \frac{\int_{-\infty}^{\infty} (F_Y (t) - F_X (t))_+^2 dt}{\int_{-\infty}^{\infty} (F_Y (t) - F_X (t))^2 dt}.
\]

Consider the metric created by the sup norm
\[
\rho_\infty(F,G) = \|F - G\|_\infty =  \sup_x |F(x) - G(x)|.
\]
Note that $T$ is continuously $\rho_\infty$-Frechet differentiable in both arguments due to the differentiability of the function $(\cdot)_+^2$ and integration. Specifically,
\begin{align*}
D_{1,(F_X, F_Y)}(G_X)  &=\frac{1}{(\int_{-\infty}^{\infty} (F_Y (t) - F_X (t))^2 dt)^2}\cdot\\&\Big[ \left(\int_{-\infty}^{\infty} (F_Y (t) - F_X (t))^2 dt\right) \left(\int_{-\infty}^\infty 2 (F_Y (t) - F_X (t))_+ G_X dt\right)\\& - \left(\int_{-\infty}^{\infty} (F_Y (t) - F_X (t))_+^2 dt\right)\left(\int_{-\infty}^\infty 2 (F_Y (t) - F_X (t)) G_X dt\right)\Big].
\end{align*}
and similarly for $D_{2,(F_X, F_Y)}(G_Y)$.
Since $T$ is continuously differentiable, by the definition of continuous Frechet differentiability we can write (see Chapter 2 in \cite{shao2012jackknife}) the following:
\begin{align*}
    &T(F^{n\ast}_X, F^{m\ast}_Y) - T(F^{n}_X, F^{m}_Y) \\&= D_{1,(F^n_X, F^m_Y)}(F_X^{n\ast} - F^n_X) + D_{2,(F^n_X, F^m_Y)}(F_Y^{m\ast} - F^m_Y) + (\rho_\infty(F_X^{n\ast}, F^n_X) + \rho_\infty (F_Y^{m\ast}, F^m_Y))\epsilon^\ast_{n,m}
\end{align*}
\begin{align*}
    &T(F^{n\ast}_X, F^{m}_Y) - T(F^{n}_X, F^{m}_Y) = D_{1,(F^n_X, F^m_Y)}(F_X^{n\ast} - F^n_X)  + (\rho_\infty(F_X^{n\ast}, F^n_X))\epsilon^\ast_{n}
\end{align*}
\begin{align*}
    &T(F^{n}_X, F^{m\ast}_Y) - T(F^{n}_X, F^{m}_Y) =  D_{2,(F^n_X, F^m_Y)}(F_Y^{m\ast} - F^m_Y) + ( \rho_\infty (F_Y^{m\ast}, F^m_Y))\epsilon^\ast_{m}
\end{align*}
and
\begin{align*}
    &T(F^{n}_X, F^{m}_Y) - T(F_X, F_Y) \\&= D_{1,(F_X, F_Y)}(F_X^{n} - F_X) + D_{2,(F_X, F_Y)}(F_Y^{m} - F_Y) + (\rho_\infty(F_X^{n}, F_X) + \rho_\infty (F_Y^{m}, F_Y))\epsilon_{n,m}
\end{align*}
\begin{align*}
    &T(F^{n}_X, F_Y) - T(F_X, F_Y) = D_{1,(F_X, F_Y)}(F_X^{n} - F_X)  + (\rho_\infty(F_X^{n}, F_X))\epsilon_{n}
\end{align*}
\begin{align*}
    &T(F_X, F^{m}_Y) - T(F_X, F_Y) =  D_{2,(F_X, F_Y)}(F_Y^{m} - F_Y) + ( \rho_\infty (F_Y^{m}, F_Y))\epsilon_{m}
\end{align*}
where $\epsilon^\ast_{n,m}, \epsilon^\ast_{n}, \epsilon^\ast_{m} ,\epsilon_{n,m}, \epsilon_{n}, \epsilon_{m}\rightarrow 0$ as $n,m \rightarrow \infty$.

Hence, combining terms,
\[
T(F^{n\ast}_X, F^{m\ast}_Y) - T(F^{n}_X, F^{m}_Y) = (T(F^{n\ast}_X, F^{m}_Y) - T(F^{n}_X, F^{m}_Y)) + (T(F^{n}_X, F^{m\ast}_Y) - T(F^{n}_X, F^{m}_Y)) + o_p(n^{-1/2} + m^{-1/2}), 
\]
and
\[
T(F^{n}_X, F^{m}_Y) - T(F_X, F_Y) = (T(F^{n}_X, F_Y) - T(F_X, F_Y)) + (T(F_X, F^{m}_Y) - T(F_X, F_Y)) + o_p(n^{-1/2} + m^{-1/2}).
\]


Hence, assuming independence of the $n$-sample and $m$-sample and respective bootstrap resamplings,
\[
\frac{\mathsf{Var}(T(F^{n\ast}_X, F^{m\ast}_Y))}{\mathsf{Var}(T(F^{n}_X, F^{m}_Y))}\rightarrow_{a.s.} \frac{\mathsf{Var}(T(F^{n}_X, F^{m\ast}_Y)) + \mathsf{Var}(T(F^{n\ast}_X, F^{m}_Y))}{\mathsf{Var}(T(F_X, F^{m}_Y)) + \mathsf{Var}(T(F^{n}_X, F_Y))},
\]
i.e. we add the variances. 

We have now divided the task to the one-sided setting where the bootstrap is only done in one argument of $T$. Hence we can apply Theorem 3.10 of \cite{shao2012jackknife} which states that for $\rho_\infty$-Frechet differentiable functions of a CDF, the bootstrap variance estimator is asymptotically consistent if the support is bounded (more general results can be stated but are omitted for simplicity). Applying separately to each of the two variances we have the following.
\begin{proposition}
    Suppose $F_X$, $F_Y$, have support contained in $[-M, M]$ for some $M > 0$, and $F_X^n$, $F_Y^m$ arise from independent samples. Then 
    \[
\frac{\mathsf{Var}(\tilde{\epsilon}_{1}(F^{n\ast}_X, F^{m\ast}_Y))}{\mathsf{Var}(\tilde{\epsilon}_{1}(F^{n}_X, F^{m}_Y))} \rightarrow_{a.s.} 1.
\]
\end{proposition}

\section{Additional Experimental Results}\label{app:AddExp}

\subsection{Statistical Significance on Synthetic Data }\label{app:stsGaussian}

We examine the statistical properties of our tests as a function of sample size.
We purposely design synthetic score distributions to represent challenging problems with large overlap between the distributions and a considerable violation ratio, but where one would still like to have an ordering among the variables.
For this we consider the two Gaussian distributions with mean $\mu=0$ and standard deviation $\sigma=1$, and with mean $\mu=0.5$ and standard deviation $\sigma=2$, respectively.
In the top panels of Figure \ref{fig:SyntTest} we show the PDF, CDF and integrated quantile function of these two Gaussians, illustrating the large violation ratio. The orange distribution can be calculated to be 0.2-FSD and 0.45-SSD over the blue distribution. Note that these $\varepsilon$ values are not comparable, due to the differences in definitions. In Figure \ref{fig:SyntTest}, we conduct experiments illustrating the power of our tests for the absolute tests of the hypotheses $H_{0,FSD} = 0.45$-FSD and  $H_{0,SSD} = 0.45$-SSD. We also use our relative tests, which in this 2-variable case (as noted in the main text) are equivalent to testing $H_{0,FSD} = 0.5$-FSD and  $H_{0,SSD} = 0.5$-SSD.
The bottom left panel in Figure \ref{fig:SyntTest} show the True Positive Rate for the different types of tests that we developed: relative test with quantile function, relative test with Integrated Quantile Function, absolute test with quantile function, and absolute test with Integrated Quantile Function.
As expected, all tests quickly converge towards True Positive Rate of 1.0 for growing sample sizes.  

\begin{figure}[ht!]
\centering
\includegraphics[width=1.0\textwidth]{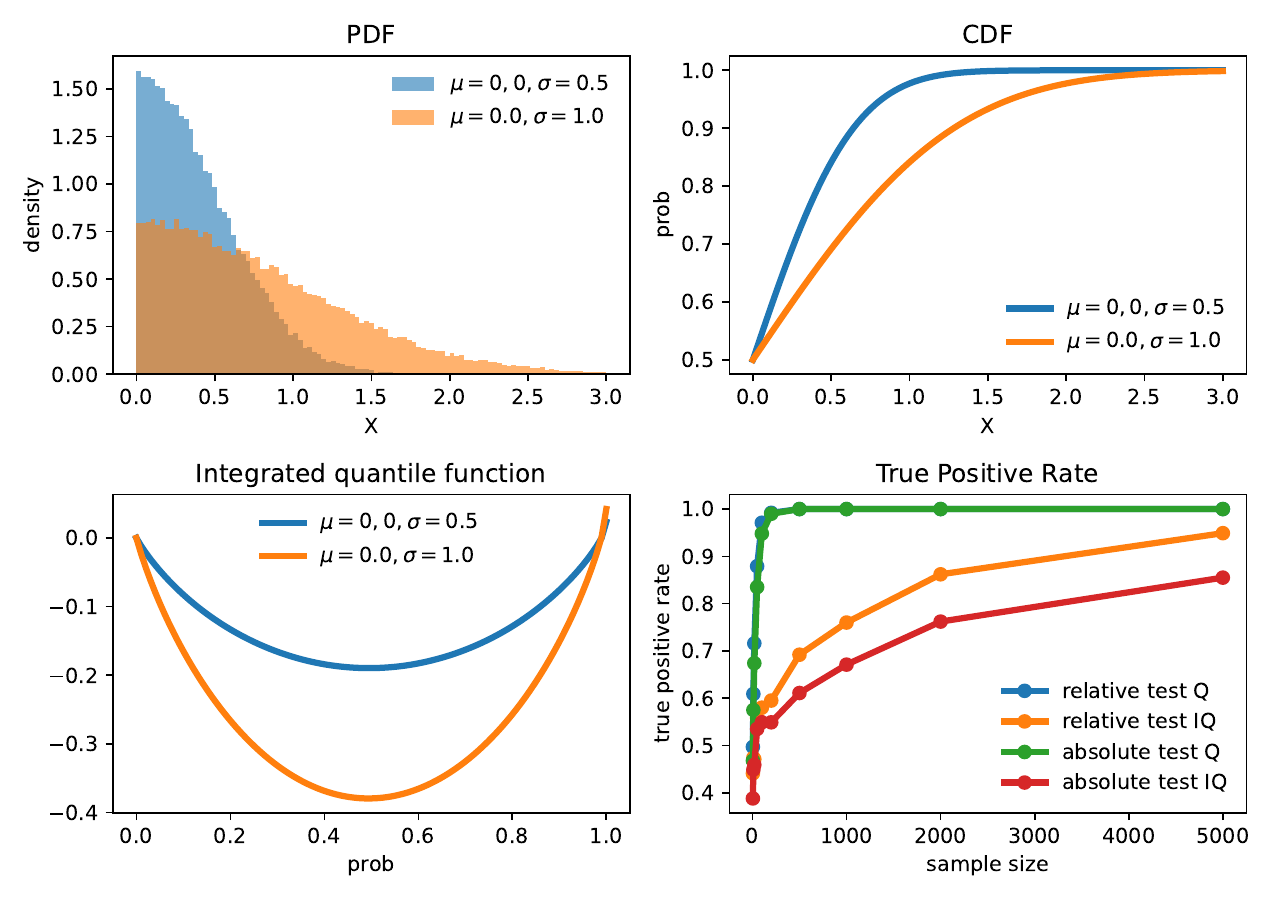} 
\caption{True Positive Rate vs sample size for \textbf{Gaussian distributions}. We compute the True Positive Rate of our stochastic dominance methods on the test distributions in the top panels for different sample sizes. Decisions are made using a confidence threshold of $\alpha=0.05$ and $\tau=0.45$ (for the absolute tests) and rates are computed over 1000 repetitions of the tests. Note that the FSD and SSD curves should not be compared due to differences in the underlying hypotheses.}
\label{fig:SyntTest}
\end{figure}

\begin{figure}[ht!]
\centering
\includegraphics[width=1.0\textwidth]{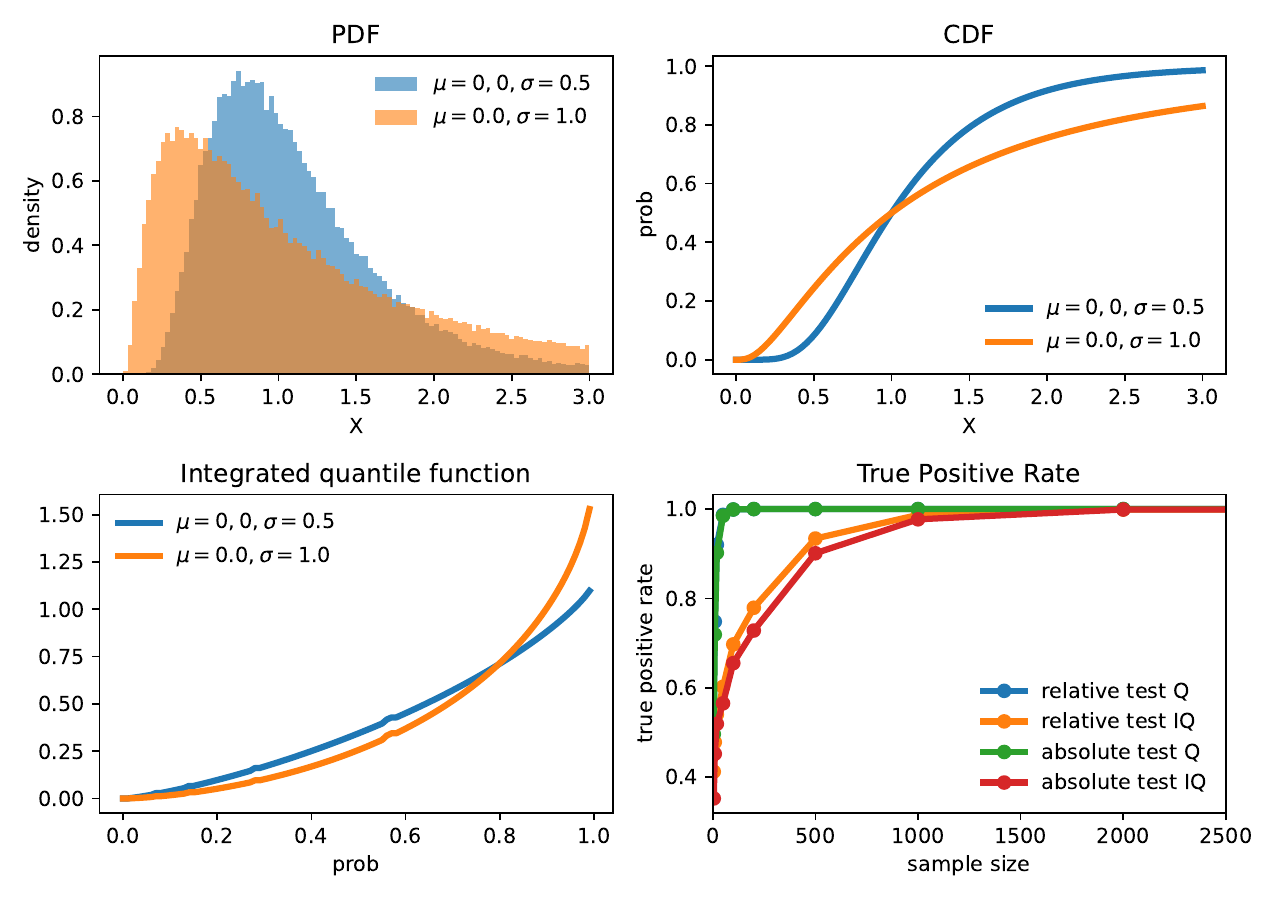}
\caption{True Positive Rate vs sample size for \textbf{Lognormal distributions} generated as $X=e^{\mu+\sigma Z}$, where $Z$ is a standard Gaussian variable. We compute the True Positive Rate of our stochastic dominance as in Fig.\ \ref{fig:SyntTest}, but in this case we examine True Positive Rate for heavy-tailed distributions examplified by Lognormal distributions.}
\label{fig:SyntTest2}
\end{figure}

\subsection{Mix-Instruct }
Results for the Mix-Instruct data are shown in Figures \ref{fig:Mixinstruct} and \ref{fig:MixinstructCDFTVAR_portfolio}, as well as Table \ref{tab:mixinstapp}.

\begin{table}[ht!]
\centering

\resizebox{\textwidth}{!}{\begin{tabular}{lcccccccccccc}
\toprule
&  Open & koala & alpaca & llama & flan-t5 &  stablelm & Vicuna & Dolly &  Moss & ChatGLM& mpt-7b & mpt-7b\\
 & assistant &  &  & 7b & &  &  & (v2)  & 6b &  &instruct\\
 \midrule
   &  & &  &  & &  & &  & &  &  & \\ 
\textbf{Mean Win Rates}   &  & &  &  & &  & &  & &  &  & \\ 
 RA(MWR @ M)    & \rankone{1} & 6 & \ranktwo{2} & 8 & 5 & 7 & \rankthree{3} & 10 &  9 & 4 & 11 & 12 \\ 
 
  &  & &  &  & &  & &  & &  &  & \\ 
MWR @P(IC)     &  \rankone{1} & 5 & \ranktwo{2} & 7 & 6 & 8 & \rankthree{3} & 9 & 10 & 4 & 11 & 12 \\ 
   &  & &  &  & &  & &  & &  &  & \\ 

\Xhline{5\arrayrulewidth}
    &  & &  &  & &  & &  & &  &  & \\ 
\textbf{Relative FSD}      &  & &  &  & &  & &  & &  &  & \\ 
RA(R-FSD @ M)    & \rankone{1} & 6 & \ranktwo{2} & 5 & 8 & 11 & 4 & 10 & 7 & \rankthree{3} & 9 & 12\\ 
    &  & &  &  & &  & &  & &  &  & \\ 
R-FSD @P(IC)   &   \rankone{1}& 6 & \ranktwo{2} & 5 & 11 & 10 & 4 & 8 & 7 & \rankthree{3} & 9 & 12\\
    &  & &  &  & &  & &  & &  &  & \\ 
R-FSD @ChatGPT   &   \rankone{1}& 7 & \rankthree{3} & 4 & 12 & 11 & \ranktwo{2} & 8 & 5 & 6 & 9 & 10\\
    &  & &  &  & &  & &  & &  &  & \\ 
\Xhline{5\arrayrulewidth}
   &  & &  &  & &  & &  & &  &  & \\ 
\textbf{Relative SSD}       &  & &  &  & &  & &  & &  &  & \\ 
    RA(R-SSD @ M) &  \rankone{1}  &  7 & \ranktwo{2} & 5 & 12 & 10 & 4 & 9 & 6 & \rankthree{3} & 8  &  11\\
   &  & &  &  & &  & &  & &  &  & \\ 
R-SSD @P(IC)     & \rankone{1} & 6 & \rankthree{3} & 5 & 12 & 11  & 4 & 7  & 8 & \ranktwo{2} & 9 & 10 \\
    &  & &  &  & &  & &  & &  &  & \\
 R-SSD @ChatGPT   &   \rankone{1}& 8 & \rankthree{3} & 4 & 11 & 12 & \ranktwo{2} & 7 & 5 & 6 & 9 & 10\\
&  & &  &  & &  & &  & &  &  & \\ 

\Xhline{5\arrayrulewidth}

 &  & &  &  & &  & &  & &  &  & \\ 
 \textbf{Mean-Risk Models}      &  & &  &  & &  & &  & &  &  & \\ 
 RA($\mu_X-\Gamma_{X}$) @ M      & \rankone{1}  & 7 & \ranktwo{2} & 5 & 12 & 11 & 4 & 9 & 6 & \rankthree{3} & 8 & 10\\ 
RA($\mu_X-r_{X}$) @P(IC)     & \rankone{1} & 6 & \rankthree{3} & 5 & 12 & 11 & 4 & 7 & 8 & \ranktwo{2} & 9 & 10\\ 
    &  & &  &  & &  & &  & &  &  & \\ 
\Xhline{5\arrayrulewidth}

\end{tabular}}
\caption{Rankings of models on following instructions according to all tests, with the top 3 ranks highlighted. We see that SSD and Mean -- Risk models are consistent. Note that RA($\mu_X-r_{X}$) @P(IC) denotes the aggregation of rankings produced by ($\mu_X-r_{X}$) @P(IC) for each $r_X$ in Table \ref{tab:RiskC}.}
\label{tab:tabMixInstructMain}
\vskip -0.15 in
\end{table}

\begin{figure}[t!]
\centering
\includegraphics[scale=0.52]{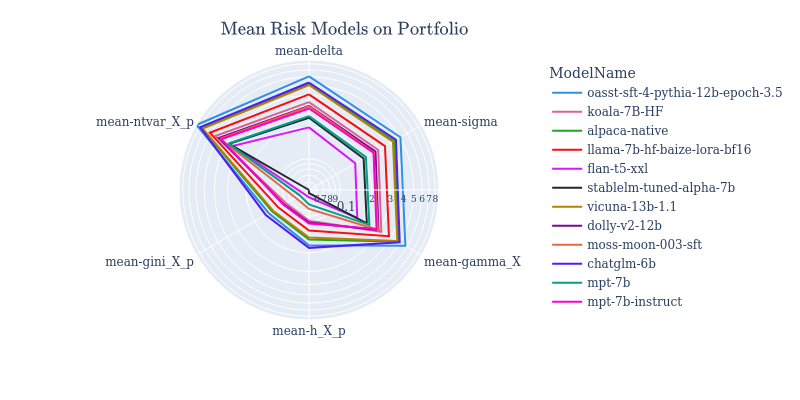} 
\caption{Radar plot of mean -- risk models of the portfolio on Mix-Instruct data. Note that the outer models are indeed the ones preferred by SSD in Table \ref{tab:mixinstapp}.}
\label{fig:Mixinstruct}
\end{figure}

\begin{figure}[ht!]
\centering
\includegraphics[scale=0.42]{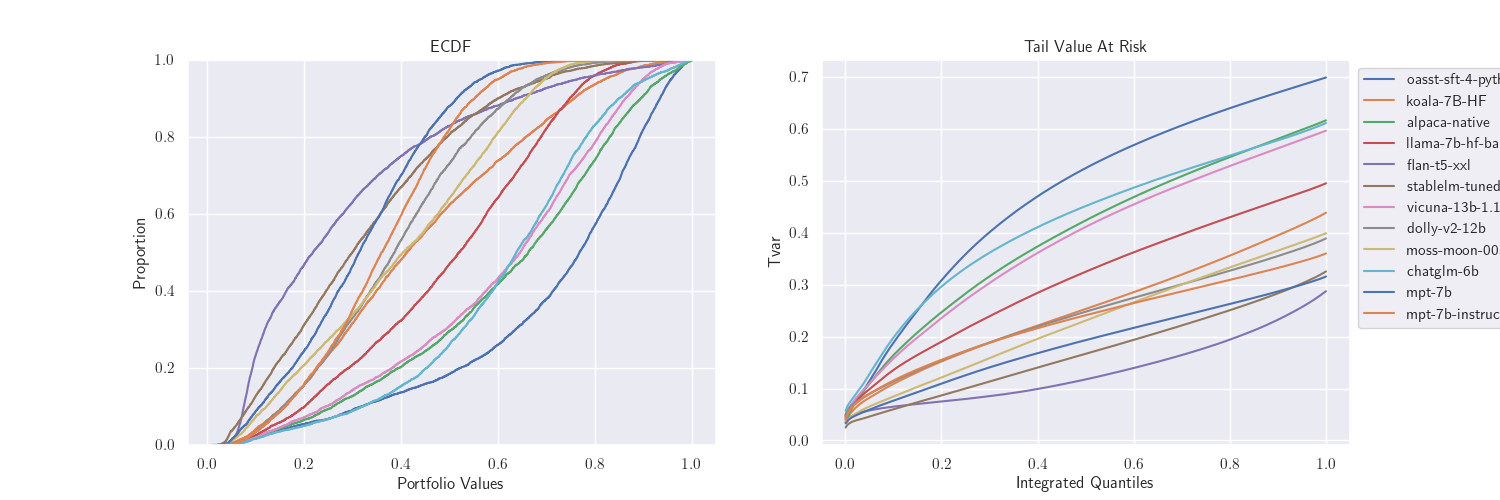} 
\caption{ Empirical CDF and TvaR for portfolio on Mix-Instruct data }
\label{fig:MixinstructCDFTVAR_portfolio}
\end{figure}

\begin{table}[ht!]
\centering

\resizebox{\textwidth}{!}{\begin{tabular}{lcccccccccccc}
\toprule
&  Open & koala & alpaca & llama & flan-t5 &  stablelm & Vicuna & Dolly &  Moss & ChatGLM& mpt-7b & mpt-7b\\
 & assistant &  &  & 7b & &  &  & (v2)  & 6b &  &instruct\\
 \midrule
   &  & &  &  & &  & &  & &  &  & \\ 
\textbf{Mean Win Rates}   &  & &  &  & &  & &  & &  &  & \\ 
  &  & &  &  & &  & &  & &  &  & \\ 
 RA(MWR @ M)    & \rankone{1} & 6 & \ranktwo{2} & 8 & 5 & 7 & \rankthree{3} & 10 &  9 & 4 & 11 & 12 \\ 
 
  &  & &  &  & &  & &  & &  &  & \\ 
MWR @P(IC)     &  \rankone{1} & 5 & \ranktwo{2} & 7 & 6 & 8 & \rankthree{3} & 9 & 10 & 4 & 11 & 12 \\ 
    &  & &  &  & &  & &  & &  &  & \\ 
   &  & &  &  & &  & &  & &  &  & \\ 
\Xhline{5\arrayrulewidth}
    &  & &  &  & &  & &  & &  &  & \\ 
\textbf{Relative FSD}      &  & &  &  & &  & &  & &  &  & \\ 
  &  & &  &  & &  & &  & &  &  & \\ 
RA(R-FSD @ M)    & \rankone{1} & 6 & \ranktwo{2} & 5 & 8 & 11 & 4 & 10 & 7 & \rankthree{3} & 9 & 12\\ 
    &  & &  &  & &  & &  & &  &  & \\ 
R-FSD @P(IC)   &   \rankone{1}& 6 & \ranktwo{2} & 5 & 11 & 10 & 4 & 8 & 7 & \rankthree{3} & 9 & 12\\
    &  & &  &  & &  & &  & &  &  & \\

    &  & &  &  & &  & &  & &  &  & \\ 

\Xhline{5\arrayrulewidth}
   &  & &  &  & &  & &  & &  &  & \\ 
\textbf{Relative SSD}       &  & &  &  & &  & &  & &  &  & \\ 
    &  & &  &  & &  & &  & &  &  & \\ 
    RA(R-SSD @ M) &  \rankone{1}  &  7 & \ranktwo{2} & 5 & 12 & 10 & 4 & 9 & 6 & \rankthree{3} & 8  &  11\\
   &  & &  &  & &  & &  & &  &  & \\ 
R-SSD @P(IC)     & \rankone{1} & 6 & \rankthree{3} & 5 & 12 & 11  & 4 & 7  & 8 & \ranktwo{2} & 9 & 10 \\
    &  & &  &  & &  & &  & &  &  & \\ 

R-SSD @ChatGPT   &   \rankone{1}& 8 & \rankthree{3} & 4 & 11 & 12 & \ranktwo{2} & 7 & 5 & 6 & 9 & 10\\
&  & &  &  & &  & &  & &  &  & \\ 
\Xhline{5\arrayrulewidth}

&  & &  &  & &  & &  & &  &  & \\ 

\textbf{Absolute FSD}       &  & &  &  & &  & &  & &  &  & \\ 
    &  & &  &  & &  & &  & &  &  & \\ 
$\varepsilon$-FSD @P(IC) $\varepsilon$=0.08    & \rankone{1} & 6 & \ranktwo{2} & 5 & 10 & 11  & 4 & 7  & 8 & \rankthree{3} & 9 & 12 \\
    &  & &  &  & &  & &  & &  &  & \\ 
$\varepsilon$-FSD @P(IC) $\varepsilon$=0.25 &  \rankone{1}  &  6 & \ranktwo{2} & 5 & 12 & 10 & 4 & 7 & 8 & \rankthree{3} & 9  &  11\\
   &  & &  &  & &  & &  & &  &  & \\ 
$\varepsilon$-FSD @P(IC) $\varepsilon$=0.4  &   \rankone{1}& 6 & \ranktwo{2} & 5 & 12 & 10 & 4 & 8 & 7 & \rankthree{3} & 9 & 11\\
&  & &  &  & &  & &  & &  &  & \\ 
\Xhline{5\arrayrulewidth}
&  & &  &  & &  & &  & &  &  & \\ 
\textbf{Absolute SSD}       &  & &  &  & &  & &  & &  &  & \\ 
    &  & &  &  & &  & &  & &  &  & \\ 
$\varepsilon$-SSD @P(IC) $\varepsilon = 0.08$     & \rankone{1} & 6 & \rankthree{3} & 5 & 12 & 11  & 4 & 7  & 8 & \ranktwo{2} & 9 & 10 \\
    &  & &  &  & &  & &  & &  &  & \\ 
$\varepsilon$-SSD @P(IC) $\varepsilon=0.25$  &  \rankone{1}  &  6 & \rankthree{3} & 5 & 12 & 11 & 4 & 8 & 7 & \ranktwo{2} & 9  & 10\\
   &  & &  &  & &  & &  & &  &  & \\ 
$\varepsilon$-SSD @P(IC) $\varepsilon$=0.4    &   \rankone{1}& 6 & \rankthree{3} & 5 & 12 & 11 & 4 & 7 & 8 & \ranktwo{2} & 9 & 10\\
&  & &  &  & &  & &  & &  &  & \\ 
\Xhline{5\arrayrulewidth}

   &  & &  &  & &  & &  & &  &  & \\ 
 \textbf{Mean-Risk Models}      &  & &  &  & &  & &  & &  &  & \\ 
    &  & &  &  & &  & &  & &  &  & \\ 

RA($\mu_X-r_{X}$) @P(IC)     & \rankone{1} & 6 & \rankthree{3} & 5 & 12 & 11 & 4 & 7 & 8 & \ranktwo{2} & 9 & 10\\ 
    &  & &  &  & &  & &  & &  &  & \\ 

    &  & &  &  & &  & &  & &  &  & \\


\Xhline{5\arrayrulewidth}


\Xhline{5\arrayrulewidth}
\end{tabular}}
\caption{ Mix instruct Extended Results.}
\label{tab:mixinstapp}
\end{table}

\subsection{Toxicity}
Toxicity results are in Table \ref{tab:toxicity}.

\begin{table}[ht!]
\centering
\resizebox{\textwidth}{!}{\begin{tabular}{lccccc}

Scenario &  Llama 2 7b & Llama 2 13b & Llama 2 70b & MosaicML MPT 30b & Tiiuae Falcon 40b \\
 &  &  &  &  &  \\
\toprule
&  &  &  &  &  \\
\textbf{Toxic Prompts}  &  &  &  &  &  \\
&  &  &  &  &  \\
RA(R-FSD @M ) (Gen Only)  & \rankthree{3} & \ranktwo{2} & 4 &\rankone{1} & 5 \\
R-FSD @P(IC)(IC)(Gen Only) & \ranktwo{2} & \rankthree{3} & 4 & \rankone{1} & 5 \\
RA(R-SSD @M ) (Gen Only)  & \rankthree{3} & \ranktwo{2} & 4 & \rankone{1} & 5\\
R-SSD@P(IC)(IC) (Gen Only) & \rankthree{3} & \ranktwo{2} & 4 & \rankone{1} & 5 \\
 &  &  &  &  &  \\
\hline
 &  &  &  &  &  \\
RA(R-FSD @M)  (Prompt + Gen) & \ranktwo{2} & \rankthree{3} & \rankone{1} & 4 & 5 \\
R-FSD @P(IC)(IC)(Prompt + Gen) & \ranktwo{2} & \rankthree{3} & \rankone{1} & 4 & 5 \\
RA(R-SSD @M)  (Prompt + Gen) & \ranktwo{2} & \rankthree{3} & \rankone{1} & 4 & 5\\
R-SSD @P(IC)(IC) (Prompt + Gen)  & \ranktwo{2} & \rankthree{3} & \rankone{1} & 4 & 5 \\
&  &  &  &  &  \\
\Xhline{5\arrayrulewidth}
&  &  &  &  &  \\
\textbf{Non-Toxic Prompts}  &  &  &  &  &  \\
&  &  &  &  &  \\
RA(R-FSD @M) (Gen Only) & \rankone{1} & \ranktwo{2} & 4 & \rankthree{3} & 5 \\
R-FSD @P(IC)(IC) (Gen Only) & \rankone{1} & \ranktwo{2} & \rankthree{3} & 4 & 5 \\
RA(R-SSD @M) (Gen Only) & \rankone{1} & \ranktwo{2} & \rankthree{3} & 4 & 5 \\
R-SSD @P(IC)(IC) (Gen Only) & \rankone{1} & \ranktwo{2} & \rankthree{3} & 4 & 5 \\
 &  &  &  &  &  \\
\hline
 &  &  &  &  &  \\

RA( R-FSD @M) (Prompt + Gen) & \rankthree{3} & \ranktwo{2} & 4 & \rankone{1} & 5 \\
 R-FSD @P(IC)  (Prompt + Gen)& \rankone{1} & \ranktwo{2} & 4 & \rankthree{3} & 5 \\
RA(R-SSD @M) (Prompt + Gen) & \rankone{1} & \ranktwo{2} & \rankthree{3} & 4 & 5\\
 R-SSD @P(IC) (Prompt + Gen) & \rankone{1} & \ranktwo{2} & 4 & \rankthree{3} & 5 \\
&  &  &  &  &  \\
\Xhline{5\arrayrulewidth}
&  &  &  &  &  \\
\textbf{All Combined (Toxic + Non-Toxic Prompts) } &  &  &  &  &  \\
&  &  &  &  &  \\
RA(R-FSD @M) (Gen Only) & \ranktwo{2} & \rankthree{3} & 5 & \rankone{1} & 4 \\
R-FSD @P(IC) (Gen Only)& \ranktwo{2} & \rankthree{3} & 5 & \rankone{1} & 4 \\
RA(R-SSD @M) (Gen Only) & \ranktwo{2} & \rankthree{3} & 5 & \rankone{1} & 4 \\
R-SSD @P(IC) (Gen Only) & \ranktwo{2} & \rankthree{3} & 5 & \rankone{1} & 4 \\

 &  &  &  &  &  \\
\hline
 &  &  &  &  &  \\
 
RA(R-FSD @M) (Prompt + Gen) & \rankthree{3} & 4 & 5 & \rankone{1} & \ranktwo{2} \\

RA(R-FSD @M) (Prompt + Gen)& \rankthree{3} & 4 & 5 & \rankone{1} & \ranktwo{2} \\

 R-SSD @P(IC) (Prompt + Gen)& \rankthree{3} & 4 & 5 & \rankone{1} & \ranktwo{2} \\
 R-SSD @P(IC) (Prompt + Gen) & \rankthree{3} & 4 & 5 & \rankone{1} & \ranktwo{2} \\
 &  &  &  &  &  \\

\Xhline{5\arrayrulewidth}



\end{tabular}}
\caption{ Toxicity Ranking Extended Results  }
\label{tab:toxicity}
\end{table}

\subsection{Ablation Study on toxicity Independent Versus Empirical Copula Portfolio} \label{app:ECvsIC}
When Comparing EC and IC portfolio aggregation using R-SSD to rank the LLM we see in Figures \ref{fig:ICversusEC40} and \ref{fig:ICversusEC20} that the two aggregation approaches lead to same ranking. While IC computational complexity is linear in the number of points , EC is quadratic. Given the correspondence in ranking IC is a more efficient aggregation technique.  
\begin{figure}[ht!]
\centering
\includegraphics[width=1.0\textwidth]{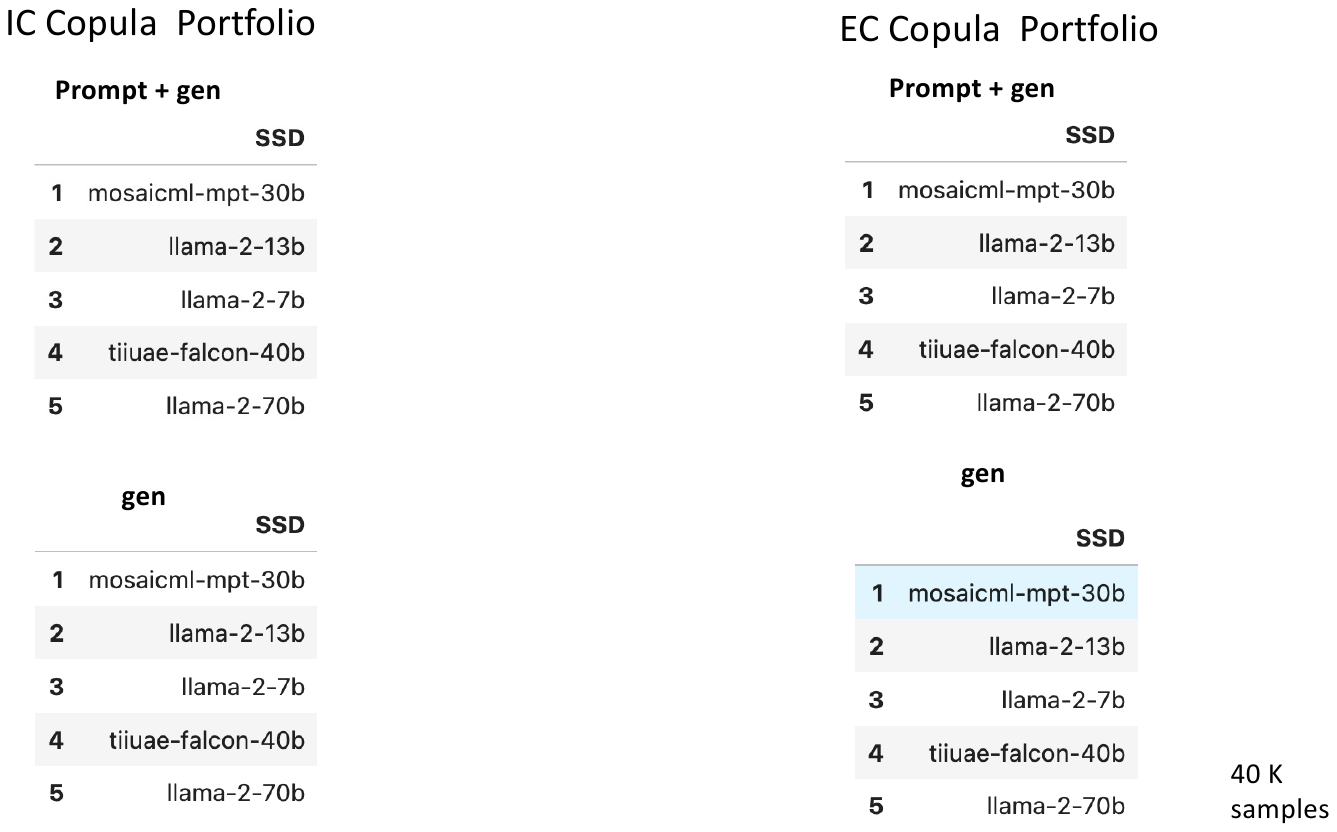}
\caption{IC versus EC Portfolio Aggregation on Toxicity. Ranking of models using 40 K samples, with independent and Empirical Copula portfolio with R-SSD. We see that the two aggregation methods lead to similar results. }
\label{fig:ICversusEC40}
\end{figure}

\begin{figure}[ht!]
\centering
\includegraphics[width=1.0\textwidth]{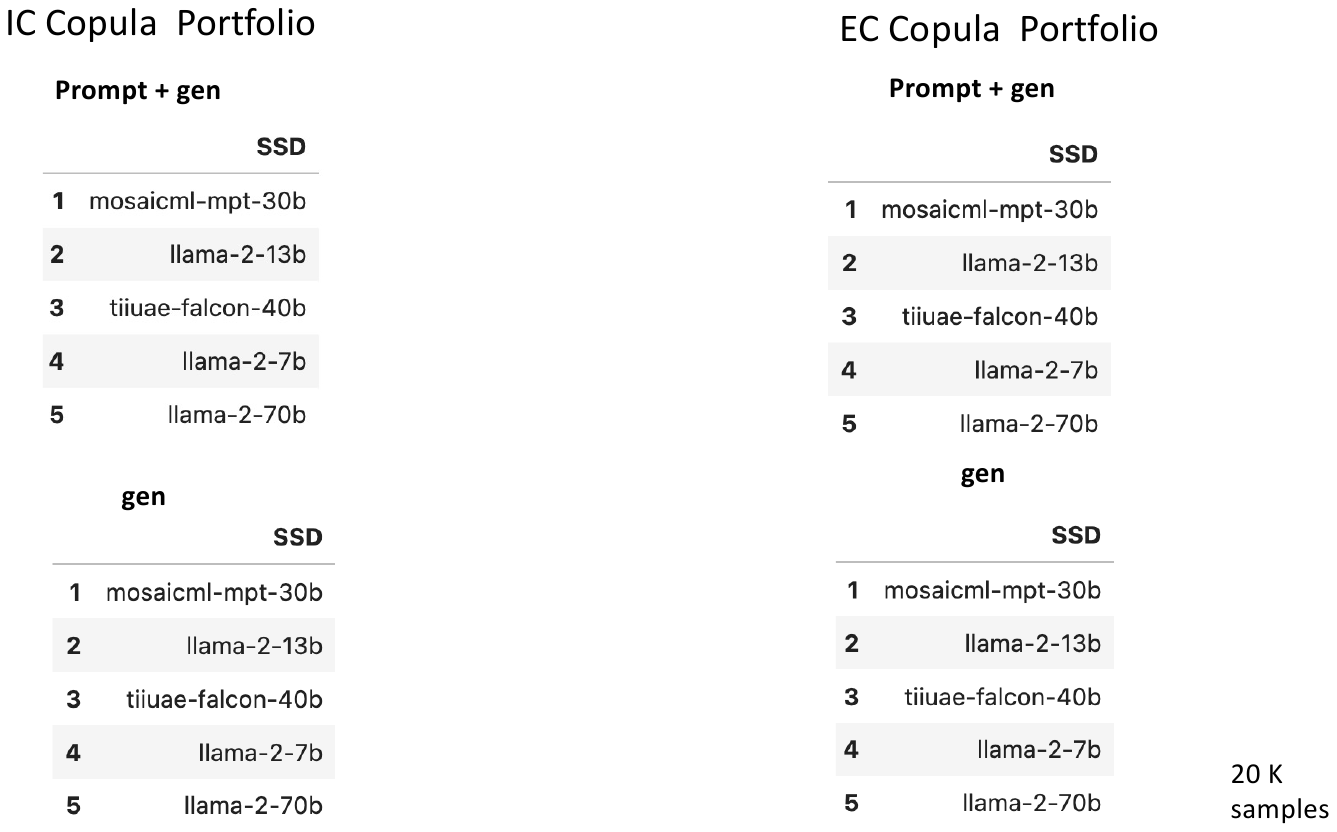}
\caption{IC versus EC Portfolio Aggregation on Toxicity. Ranking of models using 20 K samples, with independent and Empirical Copula portfolio with R-SSD. We see that the two aggregation methods lead to similar results.}
\label{fig:ICversusEC20}
\end{figure}

\subsection{Fat Left Tails of Metrics and Inconsistency  of Mean-Variance with SSD}

When metrics evaluated have fat tails, the Mean-Variance ranking can be inconsistent with the SSD. See Table \ref{tab:Mean-Var}.

\begin{table}[ht!]
\centering
\resizebox{\textwidth}{!}{\begin{tabular}{lccccc}

Scenario &  Llama 2 7b & Llama 2 13b & Llama 2 70b & MosaicML MPT 30b & Tiiuae Falcon 40b \\
&  &  &  &  &  \\
\toprule
&  &  &  &  &  \\
\textbf{Non Toxic Prompts} &  &  &  &  &  \\
&  &  &  &  &  \\
Identity Attack Metric &  &  &  &  &  \\
Gen evaluation &  &  &  &  &  \\
&  &  &  &  &  \\
 Mean - Sigma & \rankone{1} & \rankthree{3} & 4 & \ranktwo{2} & 5 \\
 Mean - Gamma  & \ranktwo{2} & \rankthree{3} & 4 & \rankone{1} & 5 \\
Mean - nTvAR   & \ranktwo{2} & \rankthree{3} & 4 & \rankone{1} & 5 \\
 SSD   & \ranktwo{2} & \rankthree{3} & 4 & \rankone{1} & 5 \\
\midrule

Threat Metric &  &  &  &  &  \\
Prompt + Gen evaluation &  &  &  &  &  \\
&  &  &  &  &  \\
 Mean - Sigma & \rankone{1} & \rankthree{3} & \ranktwo{2} & 4 & 5 \\
 Mean - Gamma  & \rankone{1} & \ranktwo{2} & \rankthree{3} & 5 & 4 \\
Mean - nTvAR   & \rankone{1} & \ranktwo{2} & \rankthree{3} & 5 & 4 \\
 SSD   & \rankone{1} & \ranktwo{2} & \rankthree{3} & 5 & 4 \\

\end{tabular}}
\caption{Inconsistency of Mean - Sigma on Toxicity Metrics with SSD and other mean-risk models. This is a due to the fact the metric evaluated may a have a fat left tail see Figures \ref{fig:Toxicity_toxic_prompt_gen_identity_attack_distribution} and \ref{fig:Toxicity_nontoxic_gen_identity_attack_distribution}. }
\label{tab:Mean-Var}
\end{table}

\begin{figure}[ht!]
\centering
\includegraphics[scale=0.6]{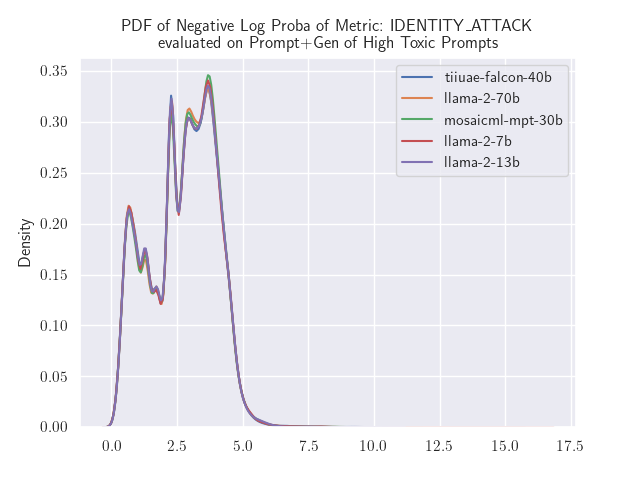} 
\caption{Identity Attack Metric distribution computed on Prompt+Generation output of Highly Toxic Prompts}
\label{fig:Toxicity_toxic_prompt_gen_identity_attack_distribution}
\end{figure}

\begin{figure}[ht!]
\centering
\includegraphics[scale=0.6]{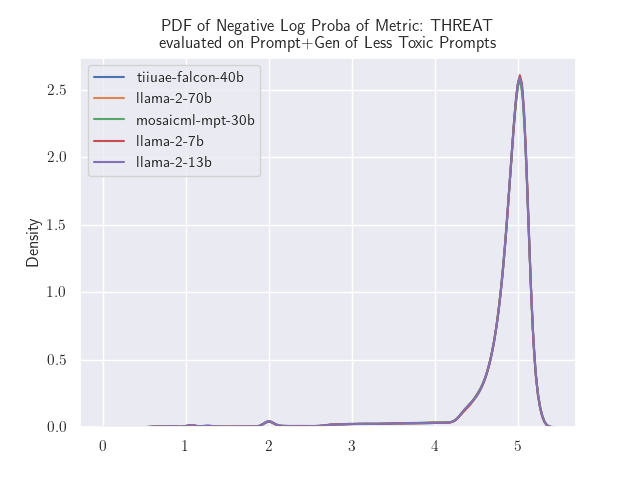} 
\caption{Threat Metric distribution computed on Prompt+Generation output of Less Toxic Prompts}
\label{fig:Toxicity_toxic_prompt_gen_threat_distribution}
\end{figure}

\begin{figure}[ht!]
\centering
\includegraphics[scale=0.6]{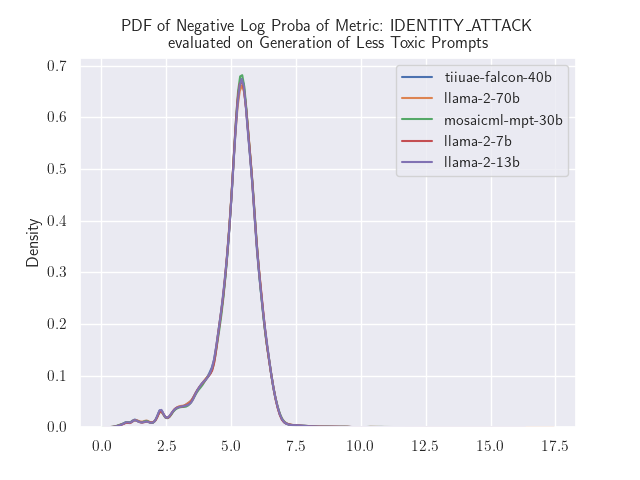} 
\caption{Identity Attack Metric distribution computed on Generation output of Less Toxic Prompts}
\label{fig:Toxicity_nontoxic_gen_identity_attack_distribution}
\end{figure}




\end{document}